\documentclass[11pt]{article}

\usepackage[final]{acl}

    \PassOptionsToPackage{numbers, compress}{natbib}

\usepackage{times}
\usepackage{latexsym}
\usepackage{bbm}

\usepackage[most]{tcolorbox}
\usepackage{enumitem}
\usepackage{inconsolata} 

\usepackage[T1]{fontenc}

\newcommand{\be}{\begin{eqnarray}}
\newcommand{\ee}{\end{eqnarray}}
\newcommand{\bi}{\begin{enumerate}}
\newcommand{\ei}{\end{enumerate}}

\usepackage[utf8]{inputenc}

\usepackage{microtype}

\usepackage{inconsolata}

\usepackage[most]{tcolorbox}
\usepackage{enumitem}

\newtcolorbox{theoremstmtbox}{
    colback=gray!10,
    colframe=gray!50,
    boxrule=0.4pt,
    arc=2pt,
    left=6pt,
    right=6pt,
    top=6pt,
    bottom=6pt,
    width=\linewidth
}

\usepackage{graphicx}
\usepackage{subcaption} 

\usepackage{url}            
\usepackage{booktabs}       
\usepackage{amsfonts}       
\usepackage{nicefrac}       
\usepackage{microtype}      
\usepackage{xcolor}         
\usepackage{enumitem}

\usepackage[table,xcdraw]{xcolor}

\usepackage[utf8]{inputenc} 
\usepackage[T1]{fontenc}    
\usepackage{hyperref}       
\usepackage{url}            
\usepackage{amsfonts}       
\usepackage{nicefrac}       
\usepackage{microtype}      
\usepackage{xcolor}         
\usepackage{makecell}
\usepackage{multirow}
\usepackage[table,xcdraw]{xcolor}
\usepackage[normalem]{ulem}
\useunder{\uline}{\ul}{}

\usepackage{microtype}
\usepackage{graphicx,enumitem}

\usepackage{array}

\usepackage{graphicx,enumitem}
\usepackage{multirow}
\usepackage{amssymb}
\usepackage{mathtools}
\usepackage{amsthm}

\usepackage{multirow,enumitem}
\usepackage{amssymb} 
\usepackage{tikz}
\usepackage{graphicx}
\usepackage{microtype}
\usepackage{multirow}
\usepackage{amssymb}
\usepackage{mathtools}
\usepackage{amsmath}
\usepackage{amsthm}
\usepackage{arydshln}
\usepackage{pifont}
\definecolor{markChanges}{rgb}{1,0,1}

\usepackage[most]{tcolorbox}  

\usepackage{hyperref}

\usepackage{algorithm}
\usepackage{algpseudocode}

\usepackage{breakurl}
\usepackage{textcomp}
\usepackage{stfloats}
\usepackage{url}
\usepackage{verbatim}
\usepackage{graphicx}
\usepackage{dsfont}

\usepackage{wrapfig}
\usepackage[dvipsnames]{xcolor}

\usepackage{dsfont}
\usepackage{pgf}
\usepackage{pgfplots}
\pgfplotsset{compat=1.18}
\usepackage{colortbl}
\usepackage{xcolor}
\usepackage{pgfplotstable}  
\usepackage{booktabs}

\usepackage{fontawesome5}

\definecolor{skyblue}{rgb}{0.529, 0.808, 0.980}
\definecolor{deepyellow}{rgb}{1.0, 0.8, 0.0}

\renewcommand{\arraystretch}{1.3}
\theoremstyle{plain}
\newtheorem{theorem}{Theorem}[section]
\newtheorem{proposition}[theorem]{Proposition}

\theoremstyle{definition}

\theoremstyle{remark}

\title{\textcolor{black}{BiG-SURE - Bipartite Graph for Semantic Uncertainty and Reliability Estimation of LLMs}}

\author{
 \textbf{Debarpan Bhattacharya\textsuperscript{*1}},
 \textbf{Malay Phadke\textsuperscript{*1}},
 \textbf{Sriram Ganapathy\textsuperscript{1,2}}
\\
 \textsuperscript{1}LEAP Lab, Electrical Engineering, Indian Institute of Science\\
 \textsuperscript{2}TANUH AI Centre of Excellence, Indian Institute of Science
\\
 \small{
   \textbf{Correspondence:} \href{mailto:email@domain}{\{debarpanb, malaymilindp, sriramg\}@iisc.ac.in}
 }
\AND
\normalfont\small
\href{https://iiscleap.github.io/projects/BiG-SURE}{\faGlobe\ \texttt{\textcolor{black}{https://iiscleap.github.io/projects/BiG-SURE}}}
\quad\textbar\quad
\href{https://github.com/iiscleap/BiG-SURE}{\faGithub\ \texttt{\textcolor{black}{https://github.com/iiscleap/BiG-SURE}}}
}

\begin{document}
\maketitle
\def\thefootnote{*}
\footnotetext{Equal contribution.}
\def\thefootnote{\arabic{footnote}}
\begin{abstract}
\textcolor{black}{
Reliable uncertainty estimation is a crucial requirement for deploying large language models (LLMs) and vision-language models (VLMs) in safety-critical settings, especially when the model parameters are not accessible (black-box). We propose \textsc{BiG-SURE}, an uncertainty estimator based on cross-temperature semantic agreement. The method samples low-temperature responses as stable semantic anchors and high-temperature responses as probes under meaning-preserving input transformations. It then constructs an anchor-probe \textbf{Bi}partite \textbf{G}raph (\textbf{BiG}) using NLI-based entailment scores and defines confidence through the normalized squared spectral energy of this matrix, with uncertainty given by its complement. \textcolor{black}{This bipartite graph-based \textbf{S}emantic \textbf{U}ncertainty and \textbf{R}eliability \textbf{E}stimation (\textbf{SURE})} score measures whether high-temperature probes remain semantically aligned with the model's stable low-temperature belief or not. We evaluate \textsc{BiG-SURE} on text QA, multilingual QA, and multimodal QA tasks across multiple model families. In these experiments, \textsc{BiG-SURE} improves average abstention AUROC over prior black-box uncertainty estimators, while remaining simple, unsupervised, and applicable to black-box model settings.
}
\end{abstract}

\section{Introduction}
The development of large language models (LLMs)-  ChatGPT~\cite{achiam2023gpt}, Gemini~\cite{team2023gemini}, Llama~\cite{touvron2023llama}, Mistral~\cite{jiang2023mistral7b}, Qwen~\cite{qwen2025qwen25technicalreport} and others -  has advanced various natural language tasks. However, multimodal and multilingual domains still underperform the English textual counterparts, affected by hallucinations and  mispredictions~\cite{wang2025lost, song2025bridge}. 
Significant efforts in alignment, grounding, and training on large-scale data are underway to address this gap.
However, in safety-critical scenarios, like finance, legal and health, the mis-predictions and hallucinations of the LLMs can result in substantial loss and hence, calibration and uncertainty measures are  the key bottlenecks that dictate  LLM deployment readiness~\cite{bedi2025testing}.
\begin{figure}[t!]
    \centering
    \includegraphics[width=0.92\linewidth, height=0.65\textwidth]{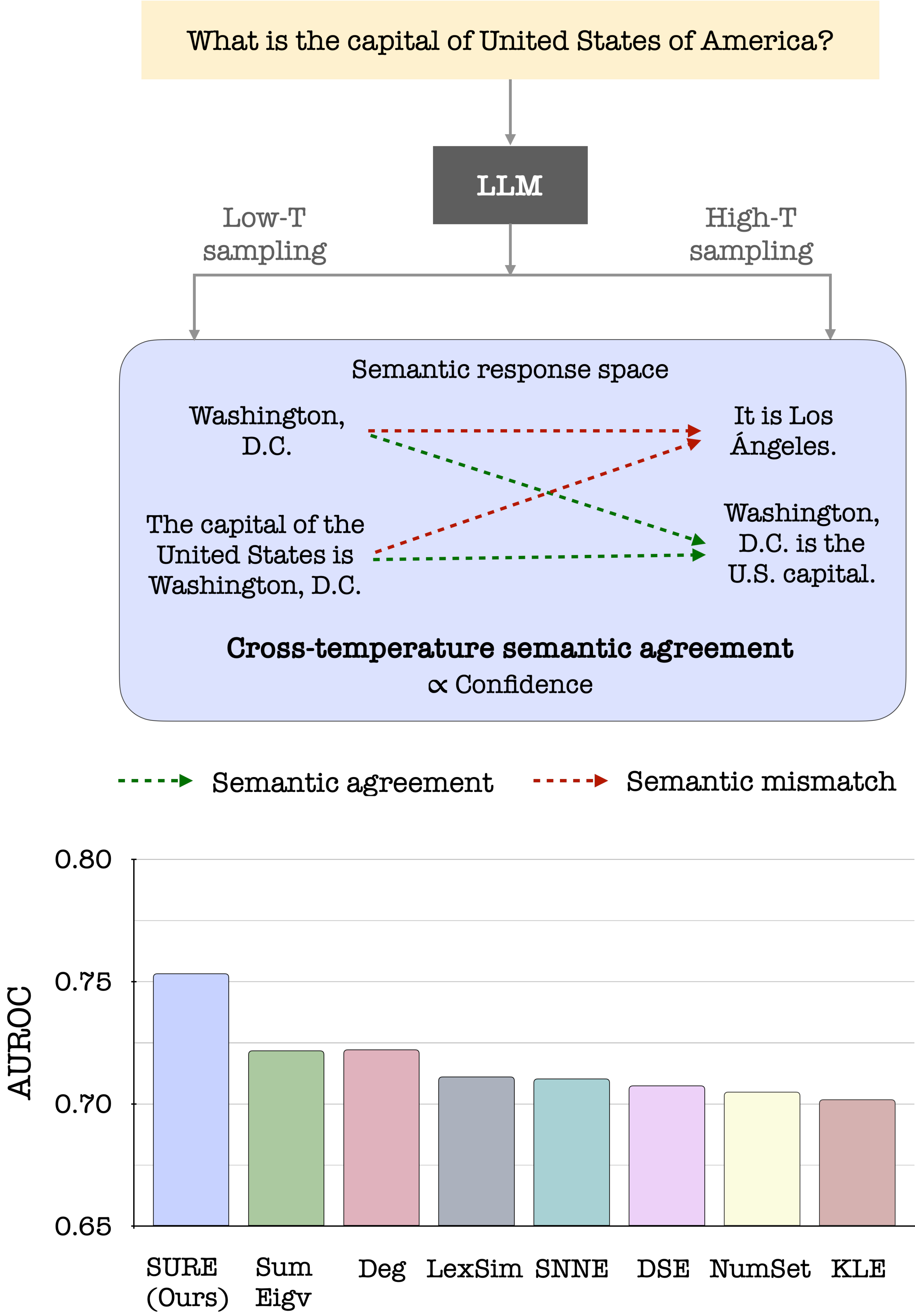}
    \caption{(Top-panel) Bipartite semantic agreement between cross-temperature responses for an example input. (Bottom-panel) SURE gives a substantial improvement in abstention AUROC over prior works on Trivia-QA. The AUROCs are averaged over $6$ LLMs (Table~\ref{tab:main_auc_all_combined}).}
    \vspace{-0.1in}
    \label{fig:intro}
\end{figure}

The early black-box approaches for uncertainty estimation focus on the semantic clustering and  entropy~\cite{kuhn2023semantic, farquhar2024detecting}. Further improvements are achieved using Von-Neumann entropy on kernelized response similarities~\cite{nikitin2024kernel}, that are devoid of the semantic clustering~\cite{nguyen2025beyond}. 
The graph-based approaches consider the responses to be different nodes of a graph and determine the effective number of semantic clusters~\cite{lin2024generating}. 
However, all the existing approaches attempt to extract the signal of uncertainty from the responses sampled at a fixed temperature.

We propose \textsc{\textbf{BiG-SURE}}, a black-box uncertainty estimation framework based on 
\emph{\textbf{cross-temperature bipartite semantic agreement}}. Low-temperature responses are used 
as stable \emph{\textbf{anchors}}, while high-temperature responses act as  \emph{\textbf{probes}} that reveal semantic 
drift. SURE constructs an anchor--probe bipartite semantic similarity graph and uses its normalized squared spectral energy as a confidence score, with uncertainty defined as its complement.
The key contributions of this work are:
\begin{itemize}[leftmargin=*, itemsep=0pt, topsep=0pt]
    \item We identify \textbf{cross-temperature response dynamics} as a useful black-box signal for uncertainty estimation in LLMs and VLMs.
    \item We propose a \textbf{bipartite anchor--probe formulation} that compares stable low-temperature generations against diverse high-temperature probes.
    \item We define SURE as a \textbf{normalized squared spectral-energy score over the anchor--probe similarity matrix, producing a bounded uncertainty estimate}.
    \item We benchmark SURE across \textbf{text-only, multilingual, and multimodal QA tasks}, showing improved abstention AUROC over prior black-box uncertainty estimators.
\end{itemize}
Figure~\ref{fig:intro} illustrates the intuition behind SURE. We sample low-temperature responses as stable \textit{anchors} and high-temperature responses as diverse \textit{probes}, then measure their bipartite semantic agreement. Reliable predictions exhibit strong cross-temperature agreement: the probes remain semantically aligned with the anchors despite surface-level variation. Uncertain or hallucinated predictions instead show semantic drift, weakening the anchor--probe agreement. SURE converts this cross-temperature agreement structure into a normalized confidence score, whose complement serves as uncertainty; the bottom panel of Figure~\ref{fig:intro} shows that this signal improves abstention performance over prior estimators.
\section{Problem Statement}
\label{sec:prob-statement}

We consider black-box uncertainty estimation for an LLM parameterized by $\theta$, where internal states, logits, and token probabilities are not accessible. Given an input $x$, the goal is to estimate an uncertainty score $\mathcal{U}(x;\theta)$ for the model's generated response $y$.

Let $T>0$ denote the sampling temperature used to generate multiple responses from the model. We define two response sets: a set of low-temperature \textit{anchor} responses 
$\mathcal{A}=\{a_1,\dots,a_M\}$, and a set of high-temperature \textit{probe} responses 
$\mathcal{H}=\{h_1,\dots,h_N\}$. The probes may be sampled either from the original input $x$ or from semantically equivalent input augmentations of $x$, such as paraphrases. Our objective is to estimate $\mathcal{U}(x;\theta)$ from the cross-temperature response sets $(\mathcal{A},\mathcal{H})$.
\section{Cross-temperature Spectral Energy}
\subsection{Definition}
We define a weighted Bipartite Graph (BiG) $G = (\mathcal{A}, \mathcal{H}, W)$ characterized by the Bi-adjacency Matrix $W \in \mathbb{R}^{M \times N}$, where   $W_{ij}$ represents the semantic similarity between anchor $a_i$ and probe $h_j$:
\begin{equation}
\label{eqn:sem-sim}
    W_{ij} = \texttt{sem-sim}(a_i, h_j), \quad W_{ij} \geq 0
\end{equation}
 We hypothesize that the cross-temperature dynamics between anchor and probe responses capture the model uncertainty. \textcolor{black}{We define the spectral energy of their agreements using $W$ as},
\begin{equation}
\label{eqn:sum_singular_values}
    E_{\texttt{SURE}}(W;\theta) = \sum_{i=1}^{min\{M,N\}}\sigma_i^2 = \sum_{i=1}^{M}\sum_{j=1}^{N}W_{ij}^2
\end{equation}
where $\sigma_i$ are singular values of $W$. The equality shows that $E_{\texttt{SURE}}(W;\theta)$ is same as the squared Frobenius norm of $W$. \textcolor{black}{
We define normalized confidence and compute the Semantic Uncertainty and Reliability Estimate (SURE) score as},
\begin{subequations}
\label{eqn:definition_unorm}
\begin{align}
    E_{\texttt{norm}}(W;\theta) &= \frac{E_{\texttt{SURE}}(W;\theta)}{MN} \\
    \mathcal{U}_{\texttt{norm}}(W;\theta) &= 1 - E_{\texttt{norm}}(W;\theta)
\end{align}
\end{subequations}
Mathematically, the measure satisfies  $0\leq \mathcal{U}_{\texttt{norm}}(W;\theta)\leq 1$ (Appendix~\ref{appendix:confidence_norm}). We formally show that SURE also possesses the properties that make it a valid uncertainty measure (Section~\ref{subsec:sure_axioms}).
\subsection{Illustration}
As shown in Figure~\ref{fig:methodology}, when an LLM generates an accurate response to a factual question (in greedy decoding mode ($T =0$)),  the mutual relationship between low and high temperature responses exhibits different properties compared to the setting when the LLM generates an inaccurate response.
 When it generates an accurate response, the low and high-temperature responses are semantically similar, which indicates a low value of $\mathcal{U}_{\texttt{norm}}$. On the other hand, when the LLM generates an incorrect response, the high-temperature responses vary significantly from the low-temperature responses, resulting in a high value of $\mathcal{U}_{\texttt{norm}}$.
 \begin{figure*}[t!]
    \centering
    \includegraphics[width=0.95\linewidth, height=0.46\textwidth]{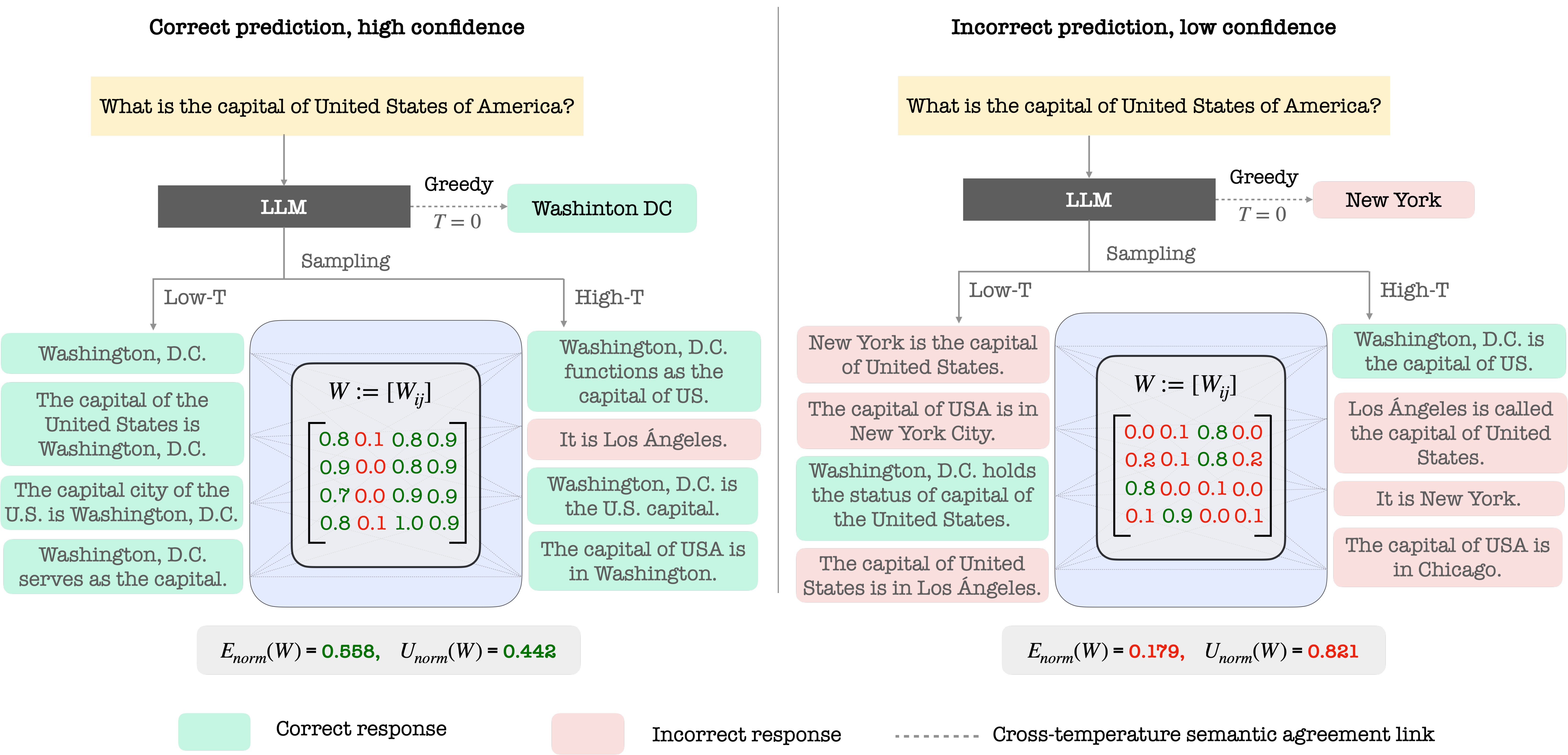}
    \caption{\textbf{Cross-temperature semantic agreement reflects prediction reliability.} 
For a correct greedy prediction (left), low- and high-temperature samples are semantically consistent, resulting in strong bipartite connectivity and high normalized spectral energy, and vice-versa.}
    \vspace{-0.1in}
    \label{fig:methodology}
\end{figure*}
 \subsection{Implementation}
Given an input prompt $x$, we sample $M$ low-temperature responses at $T_L$ and $N$ high-temperature responses at $T_H$,  with $T_H>T_L$.  
We construct a bipartite graph between cross-temperature samples, where the edge weights are based on the pairwise semantic similarity scores, and form the matrix $W\in\mathbb{R}^{M\times N}$.
We compute the SURE metric based on Eqn.~\ref{eqn:definition_unorm}. Note that the computation of uncertainty is unsupervised and does not require ground truth labels.
The algorithmic implementation is shown in Algorithm~\ref{alg:ctse}.
\begin{algorithm}[t]
\caption{Computation of $\mathcal{U}_{norm}$}
\label{alg:ctse}
\begin{algorithmic}[1]
\Require Input prompt $x$; LLM $\theta$; low temperature $T_L$; high temperature $T_H>T_L$;
\Require \# of low-temperature responses $M$; \# of high-temperature responses $N$; semantic similarity operator $\texttt{sem-sim}(\cdot,\cdot)$.
\Ensure Normalized uncertainty $\mathcal{U}_{\texttt{norm}}(W;\theta)\in[0,1]$.

\State $\mathcal A \gets \textsc{Sample}(\theta, x, T_L, M)$ 
\State $\mathcal H \gets \textsc{Sample}(\theta, x, T_H, N)$ 
\State Initialize $W \in \mathbb{R}^{M\times N}$

\For{$i \gets 1$ to $M$, $a_i\in \mathcal{A}$}
  \For{$j \gets 1$ to $N$, $h_j\in \mathcal{H}$}
    \State $W_{ij} \gets \texttt{sem-sim}(a_i, h_j)$
  \EndFor
\EndFor

\State ${E}_{\texttt{SURE}}(W;\theta) \gets \sum_{i=1}^{\min\{M,N\}} \sigma_i(W)^2$
\State $E_{\texttt{norm}}(W;\theta) \gets \dfrac{E_{\texttt{SURE}}(W;\theta)}{MN}$
\State $\mathcal{U}_{\texttt{norm}}(W;\theta) \gets 1 - E_{\mathrm{norm}}(W;\theta)$
\State \Return $\mathcal{U}_{\texttt{norm}}(W;\theta)$
\end{algorithmic}
\end{algorithm}
\subsection{SURE -  A valid uncertainty estimator}
\label{subsec:sure_axioms}
\noindent We show that $\mathcal{U}_{\texttt{norm}}(W;\theta)$ is a valid estimator of semantic uncertainty for LLMs:
\begin{itemize}[leftmargin=*, itemsep=0pt, topsep=0pt]
\item \textbf{Semantic identity:} $\mathcal{U}_{\texttt{norm}} = 0$ iff all the sampled responses of an LLM are semantically similar.
\item \textbf{Semantic disjointness:} $\mathcal{U}_{\texttt{norm}} = 1$ iff all the high temperature responses are semantically different from low temperature responses.
\item \textbf{Monotonicity:} For a given number of sampling responses, $\mathcal{U}_{\texttt{norm}}$ increases if the number of semantically unique sequences increases in the prediction set. These properties are discussed in detail in Appendix~\ref{appn:sure-valid-estimator}.
\end{itemize}


\section{Prior Works}
Several prior works attempt black-box uncertainty estimation for LLMs, like semantic entropy~\cite{kuhn2023semantic,farquhar2024detecting}, evidential semantic entropy~\cite{kunitomo2026evidential}, SPUQ~\cite{gao-etal-2024-spuq}, KLE~\cite{nikitin2024kernel}, SNNE~\cite{nguyen2025beyond}, graph-based~\cite{lin2024generating}, distillation-based~\cite{pmlr-v337-phillips26b} and Bayesian approaches~\cite{Muller_2026_CVPR, cecere2025monte}. Only a handful of recent works attempt uncertainty estimation for multimodal LLMs; some of them are FESTA~\cite{bhattacharya2025festa}, TREA~\cite{bhattacharya25b_interspeech}, Uncertainty-O~\cite{zhang2025uncertainty} and UMPIRE~\cite{lau2025uncertainty}.

\noindent \textbf{Discrete Semantic Entropy (DSE)}:
Directly computing lexical entropy~\cite{kadavath2022language} is not applicable for LLM-generated sequences, as the lexical variability can obscure semantic uncertainty. Discrete Semantic Entropy (DSE) addresses this issue by performing semantic clustering and computing entropy over the clusters~\cite{kuhn2023semantic, farquhar2024detecting}.
In contrast, our work does not require semantic clustering.\\
\noindent \textbf{Graph-based approaches}:
Another family of methods explores high-temperature sampled outputs as nodes in a graph, with edges defined by soft pairwise similarity. Various graph-derived uncertainty scores include \texttt{NumSet} (discretized similarity), \texttt{EigV} (largest eigenvalue of a degree-based matrix), and \texttt{Deg} (negative trace of the degree matrix), as studied in~\cite{lin2024generating}.
\\
\noindent \textbf{Lexical Similarity (LexSim)}:
LexSim computes a pairwise similarity matrix using a lexical overlap metric such as ROUGE-L~\cite{lin2004rouge}, and computes confidence as a normalized aggregate of off-diagonal similarities~\cite{fomicheva2020unsupervised}.\\
\noindent \textbf{Kernel Language Entropy (KLE)}:
KLE constructs unit-trace semantic kernels capturing pairwise similarity between generations, and defines uncertainty via the von Neumann entropy of the kernel matrix~\cite{nikitin2024kernel}.\\
\noindent \textbf{Semantic Nearest Neighbour Entropy (SNNE)}:
SNNE~\cite{nguyen2025beyond} avoids explicit clustering by operating directly on pairwise similarity of sampled responses. SNNE and KLE solely rely on high-temperature samples.\\
\noindent Table~\ref{tab:method_contrast_singlecol} highlights the main distinction between prior black-box uncertainty estimators and SURE. Prior approaches largely rely on fixed-temperature sampling responses and then derive uncertainty from them. SURE departs from this paradigm by using agreement dynamics between low-temperature anchors and high-temperature probes to quantify uncertainty. SURE also incorporates semantically equivalent input augmentation within the same framework. To the best of our knowledge, SURE is the first attempt to leverage cross-temperature sampling dynamics in LLM responses for uncertainty quantification.
\begin{table}[t]
\centering
\caption{Conceptual comparison of SURE with prior works on black-box, open-text uncertainty estimation.}
\label{tab:method_contrast_singlecol}
\setlength{\tabcolsep}{3.5pt}
\small


\setlength{\aboverulesep}{0pt}
\setlength{\belowrulesep}{0pt}
\setlength{\heavyrulewidth}{1.1pt}
\setlength{\lightrulewidth}{0.5pt}

\resizebox{0.9\linewidth}{!}{%
\renewcommand{\arraystretch}{1.2}
\begin{tabular}{
>{\raggedright\arraybackslash}m{1.3cm}
>{\raggedright\arraybackslash}m{2.4cm}
>{\centering\arraybackslash}m{0.9cm}
>{\raggedright\arraybackslash}m{2.6cm}
}
\specialrule{1.1pt}{0pt}{1pt}
\specialrule{0.5pt}{0pt}{0.8pt}

\textbf{Method} &
\textbf{Sampling type} &
\textbf{Input Aug.} &
\textbf{Quantification} \\

\midrule[0.7pt]

\textbf{DSE} &
High-temperature sampling (fixed $T$) &
No &
Semantic entropy over clustered responses \\

\textbf{Graph-based} &
High-temperature sampling (fixed $T$) &
No &
Graph statistics (e.g., degree, eigenvalue, number of sets) \\

\textbf{LexSim} &
High-temperature sampling (fixed $T$) &
No &
Lexical similarity across sampled responses \\

\textbf{KLE} &
High-temperature sampling (fixed $T$) &
No &
Von Neumann entropy of semantic kernel \\

\textbf{SNNE} &
High-temperature sampling (fixed $T$) &
No &
Nearest-neighbour semantic entropy \\

\rowcolor{gray!12}
\textbf{SURE} &
Low-$T$ anchors + high-$T$ probes &
Yes &
Cross-temperature agreement \\

\specialrule{0.5pt}{0.8pt}{0pt}
\specialrule{1.1pt}{0pt}{0pt}
\end{tabular}%
}
\vspace{-0.1in}
\end{table}
\section{Experimental Setup}\label{sec:experimental-setup}
\subsection{Tasks and models}
\paragraph{Textual QA:}
We experiment with text-based factual question answering and math word question answering tasks. We use Trivia-QA~\cite{joshi-etal-2017-triviaqa} and SVAMP~\cite{patel-etal-2021-nlp} datasets. We  use $6$ LLMs for these datasets - Llama-3.1-8B-Instruct and Llama-3.3-70B-Instruct ~\cite{grattafiori2024llama}, Qwen-2.5-7B-Instruct and Qwen-2.5-72B-Instruct ~\cite{qwen2025qwen25technicalreport}, Gemma-3-12b-it  and 
Gemma-3-27b-it ~\cite{gemmateam2025gemma3technicalreport}.
\paragraph{Multilingual QA:}
We experiment with multilingual SciQ ~\cite{xue-etal-2025-mlingconf} in 4 languages as English\texttt{(en)}, French\texttt{(fr)}, Japanese\texttt{(ja)} and Mandarin\texttt{(zh)}. We use two LLMs - 
Aya-expanse-8b ~\cite{dang2024ayaexpansecombiningresearch} and 
Apertus-8B-Instruct-2509~\cite{swissai2025apertus}, because of their explicit multilingual training and superior performance.
\paragraph{Multimodal QA:}
We use the OK-VQA dataset ~\cite{okvqa} for visual question-answering tasks. We use vision-language models (VLMs) - Pixtral-12B-2409 ~\cite{agrawal2024pixtral12b}, Llava-v1.6-mistral-7b-hf ~\cite{liu2023improved}  and Qwen3-VL-8B-Instruct ~\cite{bai2025qwen3vltechnicalreport}.
\subsection{Sampling and input transformations}
 \textcolor{black}{Most of the prior works on sampling-based uncertainty estimation use temperature sampling. Input transformations that do not change input semantics have also been proven to be effective ~\cite{gao-etal-2024-spuq}. SURE uses semantics-preserving input transformations in addition to temperature sampling to obtain probe responses.
 For the original input $x$,  $M$ anchor responses are obtained by low-temperature ($T_L$) sampling of $x$. $N$ probe samples are generated by high-temperature ($T_H$) sampling of a mix of equivalent input transformations, $\{x^e\}$.
 In our experiments, $M=3$, $N=10$, $T_L=0.1$ and $T_H=1.0$ are used across all models and tasks. The $N=10$ probe samples are obtained from $5$ input transformations and high-temperature sampling twice from each.
 The equivalent samples for the textual inputs correspond to paraphrasing of the input. For multimodal inputs, we use both question paraphrases and mild image perturbations: mean blur with kernel size $5$, affine rotation by $\pm 10^\circ$, and Gaussian pixel noise sampled from $\mathcal{N}(0,10)$ followed by clipping to $[0,255]$.
  The exact prompts and implementation details of these input transformations are in Appendix~\ref{app:input_aug_prompts}.
}

\subsection{\textcolor{black}{Semantic similarity backend module}}\label{sec:entailment_module}

\textcolor{black}{
The semantic similarity function $\mathrm{sem\text{-}sim}(a_i,h_j)$ in Equation~\ref{eqn:sem-sim} is implemented using NLI-based entailment module. We feed the anchor-probe pair $(a_i,h_j)$ to the entailment model and obtain the softmax over the logits. The corresponding entailment probability forms the edge weight:
}
\[
\textcolor{black}{
W_{ij}
=
\mathrm{sem\text{-}sim}(a_i,h_j)
=
p_{\mathrm{entail}}(a_i,h_j)
\in [0,1],
}
\]
\textcolor{black}{
where larger values indicate stronger semantic agreement between the anchors and probes.
For text-QA and multimodal-QA, we use a DeBERTa-large NLI model~\cite{he2021deberta}~\footnote{\url{https://huggingface.co/microsoft/deberta-large-mnli}} as the entailment backend. For multilingual QA, we use a multilingual DeBERTa NLI model (mDeBERTa~\footnote{\url{https://huggingface.co/MoritzLaurer/mDeBERTa-v3-base-xnli-multilingual-nli-2mil7}}), which is
trained on multilingual NLI data. We do not train or fine-tune
either semantic-similarity model. 
 Within each experimental setting, the entailment backend remains fixed across SURE and all similarity-based baselines for fair comparison.
}
\subsection{\textcolor{black}{Evaluation and performance metric}}
We use the area under the receiver operating characteristic curve (AUROC) to evaluate the uncertainty estimators. Each uncertainty score is interpreted as an \emph{abstention score}. First, every test prediction is labeled as correct or incorrect based on the match between the greedy prediction ($T=0$ response) and the ground-truth answer. \textcolor{black}{The matching is based on the standard evaluation protocols for textual QA and multimodal QA tasks, and we use \texttt{Gemini-2.5-Flash}~\cite{comanici2025gemini} as a judge for multilingual QA (details in Appendix~\ref{appendix:evaluation})}. Then, the binary matching labels and abstention scores are used to compute the AUROC.
\section{Results}
\subsection{Illustrative examples }
Figure~\ref{fig:examples_all} provides examples of cross-temperature response patterns captured by SURE. Across multilingual and multimodal QA examples, we observe a clear separation between reliable and unreliable predictions in terms of cross-temperature semantic agreement dynamics.
 \begin{figure*}[t!]
    \centering
    \includegraphics[width=1.0\linewidth, height=0.44\textwidth]{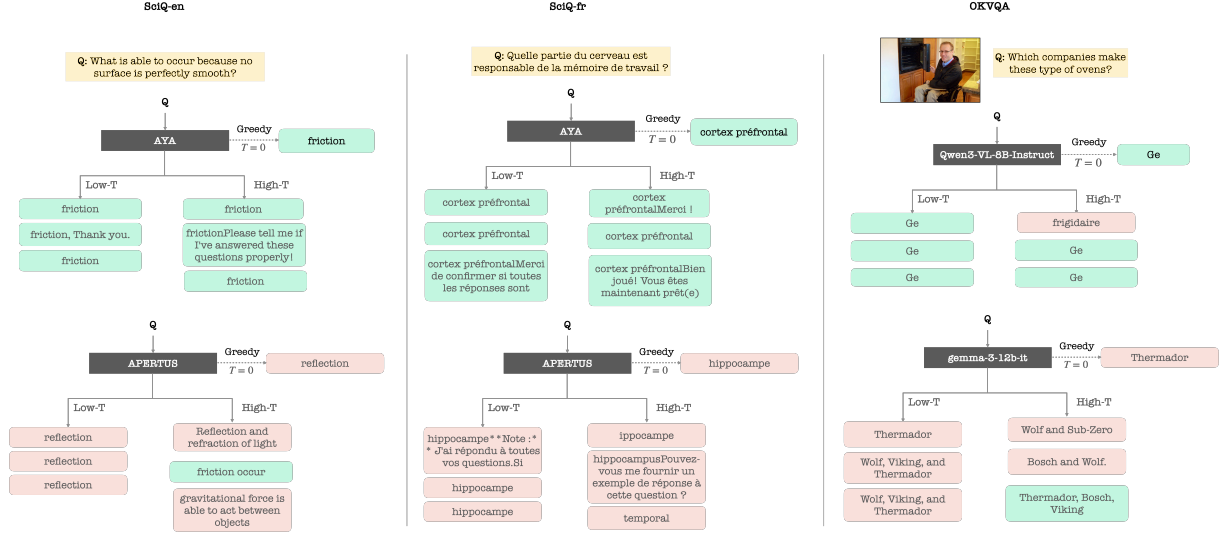}
    \caption{\textbf{Illustrative examples of cross-temperature semantic agreement in multilingual and multimodal QA.} For each example, we show the greedy prediction together with low-temperature anchor samples and high-temperature probe samples.  Responses aligned with the ground truth are marked in \textcolor{green}{green}, while responses that do not align are highlighted in \textcolor{red}{red}. Reliable predictions exhibit strong semantic agreement across temperatures, while unreliable predictions show semantic drift and disagreement.}
    \label{fig:examples_all}
\end{figure*}
The quantitative evaluation of SURE in terms of abstention AUROC and a comparison with $7$ existing works on black-box, open-text uncertainty quantification methods across multilingual and multimodal tasks is reported in Table~\ref{tab:main_auc_all_combined}. All prior works use $N=10$ in Table~\ref{tab:main_auc_all_combined}.
\subsection{Results on text-only QA tasks}
\textcolor{black}{
On the text-only QA benchmarks, SURE achieves the best average AUROC on both datasets. On \textsc{Trivia-QA}, SURE obtains an average AUROC of $0.753 \pm 0.017$, improving over the strongest baseline average, \texttt{Deg} ($0.722 \pm 0.007$). On \textsc{SVAMP}, SURE obtains $0.877 \pm 0.022$, outperforming the strongest baseline average, \texttt{SNNE} ($0.846 \pm 0.018$).}\\
\noindent \textcolor{black}{The gains are also consistent across model families and scales. On \textsc{Trivia-QA}, SURE achieves the best performance for $5$ out of $6$ models, with the only exception being Llama-70B, where \texttt{Deg} performs better. On \textsc{SVAMP}, SURE achieves the best results for $5$ out of $6$ models, with Qwen-72B being the exception where \texttt{SumEigv} is strongest. These results indicate that cross-temperature anchor--probe agreement provides a robust uncertainty signal across both factual QA and math word-problem settings. Graph-based baselines such as \texttt{SumEigv} and \texttt{Deg} are competitive, especially on \textsc{Trivia-QA}, while \texttt{SNNE} is a competitive baseline on \textsc{SVAMP}; however, SURE achieves the best average performance across both. To understand the corresponding abstention behavior, AURC values are also reported in Appendix~\ref{appendix:aurc-results}.}
\subsection{Results on multilingual QA tasks}
\textcolor{black}{
On multilingual \textsc{SciQ}, SURE performs well across languages and LLMs, and achieves the best average AUROC on English ($0.677 \pm 0.015$), Japanese ($0.685 \pm 0.013$), and Mandarin ($0.691 \pm 0.012$). These results improve over the best baseline systems in  three language settings.
French \textsc{SciQ} is the only exception, where \texttt{LexSim} obtains the best average AUROC.  Overall, SURE achieves the best average performance in $3$ out of $4$ multilingual language groups, and is especially consistent on Japanese and Mandarin for both Aya and Apertus.
}
\subsection{Results on multimodal QA tasks}
\textcolor{black}{
On OKVQA, SURE achieves the best average AUROC of $0.752 \pm 0.014$, outperforming the best baseline system, SNNE ($0.735 \pm 0.016$). SURE is also the best method for all three  VLMs. This suggests that cross-temperature semantic agreement remains useful even when the input contains visual modalities.
}\\

\noindent\textbf{Summary.}
\textcolor{black}{Across the seven dataset/language groups in Table~\ref{tab:main_auc_all_combined}, SURE achieves the best average AUROC in six groups: \textsc{Trivia-QA}, \textsc{SVAMP}, \textsc{SciQ} \texttt{en}, \textsc{SciQ} \texttt{ja}, \textsc{SciQ} \texttt{zh}, and \textsc{OKVQA}. Overall, they support the central hypothesis that cross-temperature semantic agreement provides an effective black-box uncertainty signal across text-only, multilingual, and multimodal QA settings.}
\begin{table*}[t]
\centering
\caption{\textcolor{black}{Performance comparison of SURE with prior works in terms of abstention AUROC across text-only, multimodal, and multilingual QA datasets. The third column reports the task accuracy of greedy responses. The abstention AUROCs are presented with repeated-run statistics (mean $\pm$ std) across $5$ seeds. The best result, in terms of mean AUROC, is highlighted in \textbf{bold}, while the second best result is \underline{underlined}.}}
\label{tab:main_auc_all_combined}
\setlength{\tabcolsep}{5pt}

\setlength{\aboverulesep}{0pt}
\setlength{\belowrulesep}{0pt}
\setlength{\heavyrulewidth}{1.2pt}
\setlength{\lightrulewidth}{0.55pt}

\resizebox{1.0\linewidth}{!}{%
\renewcommand{\arraystretch}{1.22}%
\begin{tabular}{l l c *{7}{c} c}

\specialrule{1.2pt}{0pt}{1.2pt}
\specialrule{0.55pt}{0pt}{0.9pt}

\multirow{2}{*}{\textbf{\makecell{Dataset\\{[}lang{]}}}} &
\multirow{2}{*}{\textbf{Model}} &
\multirow{2}{*}{\textbf{Acc.}} &
\multicolumn{7}{c}{\textbf{Baselines ($\uparrow$)}} &
\multirow{2}{*}{\textbf{SURE ($\uparrow$)}} \\[4pt]
\cmidrule(lr){4-10}
& & &
\textbf{DSE} &
\textbf{SumEigv} &
\textbf{Deg} &
\textbf{NumSet} &
\textbf{LexSim} &
\textbf{KLE} &
\textbf{SNNE} &
\\[4pt]

\midrule[0.8pt]

\multirow{7}{*}{\textbf{\makecell{Trivia-QA\\{[}en{]}}}}
& \textbf{Qwen (7B)}   & 0.495 & 0.786 $\pm$ 0.003 & 0.791 $\pm$ 0.005 & \underline{0.793 $\pm$ 0.004} & 0.784 $\pm$ 0.005 & 0.778 $\pm$ 0.008 & 0.773 $\pm$ 0.007 & 0.783 $\pm$ 0.003 & \textbf{0.806 $\pm$ 0.014} \\
& \textbf{Llama (8B)}  & 0.661 & 0.731 $\pm$ 0.005 & 0.730 $\pm$ 0.005 & \underline{0.740 $\pm$ 0.005} & 0.724 $\pm$ 0.007 & 0.736 $\pm$ 0.006 & 0.722 $\pm$ 0.003 & 0.736 $\pm$ 0.006 & \textbf{0.772 $\pm$ 0.021} \\
& \textbf{Gemma (12B)} & 0.643 & 0.698 $\pm$ 0.010 & 0.697 $\pm$ 0.011 & 0.695 $\pm$ 0.011 & 0.700 $\pm$ 0.010 & \underline{0.727 $\pm$ 0.014} & 0.689 $\pm$ 0.008 & 0.713 $\pm$ 0.014 & \textbf{0.789 $\pm$ 0.011} \\
& \textbf{Gemma (27B)} & 0.755 & 0.629 $\pm$ 0.005 & 0.636 $\pm$ 0.005 & 0.634 $\pm$ 0.005 & 0.630 $\pm$ 0.005 & \underline{0.657 $\pm$ 0.005} & 0.629 $\pm$ 0.006 & 0.617 $\pm$ 0.010 & \textbf{0.699 $\pm$ 0.017} \\
& \textbf{Llama (70B)} & 0.767 & 0.712 $\pm$ 0.008 & \underline{0.754 $\pm$ 0.011} & \textbf{0.756 $\pm$ 0.011} & 0.712 $\pm$ 0.008 & 0.737 $\pm$ 0.005 & 0.702 $\pm$ 0.010 & 0.715 $\pm$ 0.007 & 0.715 $\pm$ 0.009 \\
& \textbf{Qwen (72B)}  & 0.690 & 0.689 $\pm$ 0.012 & \underline{0.722 $\pm$ 0.005} & 0.716 $\pm$ 0.007 & 0.680 $\pm$ 0.013 & 0.631 $\pm$ 0.009 & 0.695 $\pm$ 0.005 & 0.698 $\pm$ 0.012 & \textbf{0.740 $\pm$ 0.030} \\
\multicolumn{1}{c}{} & \cellcolor{gray!12}\textbf{Avg.}
& \cellcolor{gray!12}0.668
& \cellcolor{gray!12}0.708 $\pm$ 0.007
& \cellcolor{gray!12}0.722 $\pm$ 0.007
& \cellcolor{gray!12}\underline{0.722 $\pm$ 0.007}
& \cellcolor{gray!12}0.705 $\pm$ 0.008
& \cellcolor{gray!12}0.711 $\pm$ 0.008
& \cellcolor{gray!12}0.702 $\pm$ 0.006
& \cellcolor{gray!12}0.710 $\pm$ 0.009
& \cellcolor{gray!12}\textbf{0.753 $\pm$ 0.017} \\

\midrule[0.8pt]

\multirow{7}{*}{\textbf{\makecell{SVAMP\\{[}en{]}}}}
& \textbf{Qwen (7B)}   & 0.688 & 0.794 $\pm$ 0.027 & 0.790 $\pm$ 0.019 & 0.789 $\pm$ 0.019 & 0.791 $\pm$ 0.027 & 0.780 $\pm$ 0.020 & 0.788 $\pm$ 0.018 & \underline{0.803 $\pm$ 0.029} & \textbf{0.804 $\pm$ 0.017} \\
& \textbf{Llama (8B)}  & 0.672 & 0.857 $\pm$ 0.022 & 0.844 $\pm$ 0.025 & 0.845 $\pm$ 0.025 & 0.855 $\pm$ 0.022 & 0.845 $\pm$ 0.016 & 0.861 $\pm$ 0.019 & \underline{0.888 $\pm$ 0.020} & \textbf{0.890 $\pm$ 0.026} \\
& \textbf{Gemma (12B)} & 0.730 & 0.779 $\pm$ 0.016 & 0.787 $\pm$ 0.012 & 0.790 $\pm$ 0.013 & 0.779 $\pm$ 0.015 & 0.783 $\pm$ 0.013 & 0.781 $\pm$ 0.012 & \underline{0.794 $\pm$ 0.015} & \textbf{0.882 $\pm$ 0.021} \\
& \textbf{Gemma (27B)} & 0.794 & 0.735 $\pm$ 0.014 & 0.751 $\pm$ 0.015 & 0.752 $\pm$ 0.015 & 0.735 $\pm$ 0.014 & 0.734 $\pm$ 0.015 & 0.733 $\pm$ 0.015 & \underline{0.796 $\pm$ 0.016} & \textbf{0.879 $\pm$ 0.017} \\
& \textbf{Llama (70B)} & 0.797 & 0.884 $\pm$ 0.011 & 0.882 $\pm$ 0.011 & 0.884 $\pm$ 0.011 & 0.882 $\pm$ 0.010 & 0.868 $\pm$ 0.009 & 0.881 $\pm$ 0.013 & \underline{0.896 $\pm$ 0.010} & \textbf{0.905 $\pm$ 0.013} \\
& \textbf{Qwen (72B)}  & 0.849 & 0.893 $\pm$ 0.021 & \textbf{0.915 $\pm$ 0.020} & \underline{0.914 $\pm$ 0.018} & 0.880 $\pm$ 0.022 & 0.800 $\pm$ 0.014 & 0.814 $\pm$ 0.027 & 0.897 $\pm$ 0.016 & 0.902 $\pm$ 0.039 \\
\multicolumn{1}{c}{} & \cellcolor{gray!12}\textbf{Avg.}
& \cellcolor{gray!12}0.755
& \cellcolor{gray!12}0.823 $\pm$ 0.018
& \cellcolor{gray!12}0.828 $\pm$ 0.017
& \cellcolor{gray!12}0.829 $\pm$ 0.017
& \cellcolor{gray!12}0.820 $\pm$ 0.018
& \cellcolor{gray!12}0.802 $\pm$ 0.015
& \cellcolor{gray!12}0.810 $\pm$ 0.017
& \cellcolor{gray!12}\underline{0.846 $\pm$ 0.018}
& \cellcolor{gray!12}\textbf{0.877 $\pm$ 0.022} \\

\midrule[0.8pt]

\multirow{3}{*}{\textbf{\makecell{SciQ\\{[}en{]}}}}
& \textbf{Aya} & 0.786 & 0.591 $\pm$ 0.023 & 0.603 $\pm$ 0.040 & \underline{0.613 $\pm$ 0.028} & 0.562 $\pm$ 0.026 & 0.591 $\pm$ 0.012 & 0.580 $\pm$ 0.008 & 0.537 $\pm$ 0.043 & \textbf{0.633 $\pm$ 0.004} \\
& \textbf{Apertus} & 0.763 & 0.666 $\pm$ 0.047 & \underline{0.682 $\pm$ 0.022} & 0.679 $\pm$ 0.014 & 0.603 $\pm$ 0.031 & 0.620 $\pm$ 0.020 & 0.631 $\pm$ 0.009 & 0.584 $\pm$ 0.035 & \textbf{0.722 $\pm$ 0.026} \\
\multicolumn{1}{c}{} & \cellcolor{gray!12}\textbf{Avg.}
& \cellcolor{gray!12}0.775
& \cellcolor{gray!12}0.628 $\pm$ 0.035
& \cellcolor{gray!12}0.642 $\pm$ 0.031
& \cellcolor{gray!12}\underline{0.646 $\pm$ 0.021}
& \cellcolor{gray!12}0.582 $\pm$ 0.028
& \cellcolor{gray!12}0.606 $\pm$ 0.016
& \cellcolor{gray!12}0.606 $\pm$ 0.008
& \cellcolor{gray!12}0.561 $\pm$ 0.039
& \cellcolor{gray!12}\textbf{0.677 $\pm$ 0.015} \\

\midrule[0.8pt]

\multirow{3}{*}{\textbf{\makecell{SciQ\\{[}fr{]}}}}
& \textbf{Aya} & 0.675 & \textbf{0.647 $\pm$ 0.028} & 0.631 $\pm$ 0.017 & 0.639 $\pm$ 0.015 & 0.634 $\pm$ 0.014 & \underline{0.647 $\pm$ 0.013} & 0.611 $\pm$ 0.024 & 0.635 $\pm$ 0.015 & 0.639 $\pm$ 0.015 \\
& \textbf{Apertus} & 0.668 & 0.535 $\pm$ 0.022 & \textbf{0.573 $\pm$ 0.013} & 0.563 $\pm$ 0.017 & 0.528 $\pm$ 0.032 & \underline{0.569 $\pm$ 0.018} & 0.527 $\pm$ 0.032 & 0.534 $\pm$ 0.039 & 0.542 $\pm$ 0.022 \\
\multicolumn{1}{c}{} & \cellcolor{gray!12}\textbf{Avg.}
& \cellcolor{gray!12}0.672
& \cellcolor{gray!12}0.591 $\pm$ 0.025
& \cellcolor{gray!12}\underline{0.602 $\pm$ 0.015}
& \cellcolor{gray!12}0.601 $\pm$ 0.016
& \cellcolor{gray!12}0.581 $\pm$ 0.023
& \cellcolor{gray!12}\textbf{0.608 $\pm$ 0.016}
& \cellcolor{gray!12}0.569 $\pm$ 0.028
& \cellcolor{gray!12}0.585 $\pm$ 0.027
& \cellcolor{gray!12}0.590 $\pm$ 0.019 \\

\midrule[0.8pt]

\multirow{3}{*}{\textbf{\makecell{SciQ\\{[}ja{]}}}}
& \textbf{Aya} & 0.539 & \underline{0.586 $\pm$ 0.017} & 0.564 $\pm$ 0.033 & 0.573 $\pm$ 0.025 & 0.546 $\pm$ 0.027 & 0.508 $\pm$ 0.009 & 0.558 $\pm$ 0.030 & 0.541 $\pm$ 0.030 & \textbf{0.662 $\pm$ 0.010} \\
& \textbf{Apertus} & 0.441 & 0.658 $\pm$ 0.007 & 0.678 $\pm$ 0.025 & \underline{0.690 $\pm$ 0.015} & 0.639 $\pm$ 0.002 & 0.503 $\pm$ 0.014 & 0.678 $\pm$ 0.015 & 0.625 $\pm$ 0.016 & \textbf{0.707 $\pm$ 0.016} \\
\multicolumn{1}{c}{} & \cellcolor{gray!12}\textbf{Avg.}
& \cellcolor{gray!12}0.490
& \cellcolor{gray!12}0.622 $\pm$ 0.012
& \cellcolor{gray!12}0.621 $\pm$ 0.029
& \cellcolor{gray!12}\underline{0.632 $\pm$ 0.020}
& \cellcolor{gray!12}0.593 $\pm$ 0.015
& \cellcolor{gray!12}0.506 $\pm$ 0.011
& \cellcolor{gray!12}0.618 $\pm$ 0.022
& \cellcolor{gray!12}0.583 $\pm$ 0.023
& \cellcolor{gray!12}\textbf{0.685 $\pm$ 0.013} \\

\midrule[0.8pt]

\multirow{3}{*}{\textbf{\makecell{SciQ\\{[}zh{]}}}}
& \textbf{Aya} & 0.556 & \underline{0.602 $\pm$ 0.035} & 0.592 $\pm$ 0.028 & 0.593 $\pm$ 0.031 & 0.553 $\pm$ 0.034 & 0.500 $\pm$ 0.007 & 0.568 $\pm$ 0.021 & 0.552 $\pm$ 0.028 & \textbf{0.665 $\pm$ 0.012} \\
& \textbf{Apertus} & 0.458 & 0.678 $\pm$ 0.017 & 0.689 $\pm$ 0.012 & 0.695 $\pm$ 0.016 & 0.645 $\pm$ 0.011 & 0.503 $\pm$ 0.009 & \underline{0.695 $\pm$ 0.018} & 0.613 $\pm$ 0.016 & \textbf{0.717 $\pm$ 0.012} \\
\multicolumn{1}{c}{} & \cellcolor{gray!12}\textbf{Avg.}
& \cellcolor{gray!12}0.507
& \cellcolor{gray!12}0.640 $\pm$ 0.026
& \cellcolor{gray!12}0.640 $\pm$ 0.020
& \cellcolor{gray!12}\underline{0.644 $\pm$ 0.023}
& \cellcolor{gray!12}0.599 $\pm$ 0.023
& \cellcolor{gray!12}0.502 $\pm$ 0.008
& \cellcolor{gray!12}0.632 $\pm$ 0.019
& \cellcolor{gray!12}0.583 $\pm$ 0.022
& \cellcolor{gray!12}\textbf{0.691 $\pm$ 0.012} \\

\midrule[0.8pt]

\multirow{4}{*}{\textbf{\makecell{OKVQA\\{[}en{]}}}}
& \textbf{Pixtral} & 0.630 & 0.731 $\pm$ 0.018 & \underline{0.750 $\pm$ 0.013} & 0.739 $\pm$ 0.017 & 0.740 $\pm$ 0.017 & 0.728 $\pm$ 0.022 & 0.742 $\pm$ 0.017 & 0.743 $\pm$ 0.018 & \textbf{0.775 $\pm$ 0.008} \\
& \textbf{LLaVa} & 0.643 & 0.740 $\pm$ 0.012 & 0.733 $\pm$ 0.007 & 0.755 $\pm$ 0.008 & 0.724 $\pm$ 0.013 & 0.738 $\pm$ 0.016 & 0.742 $\pm$ 0.011 & \underline{0.755 $\pm$ 0.013} & \textbf{0.756 $\pm$ 0.022} \\
& \textbf{Qwen} & 0.637 & 0.695 $\pm$ 0.012 & 0.711 $\pm$ 0.010 & 0.705 $\pm$ 0.012 & 0.693 $\pm$ 0.012 & \underline{0.718 $\pm$ 0.010} & 0.711 $\pm$ 0.009 & 0.707 $\pm$ 0.016 & \textbf{0.725 $\pm$ 0.011} \\
\multicolumn{1}{c}{} & \cellcolor{gray!12}\textbf{Avg.}
& \cellcolor{gray!12}0.637
& \cellcolor{gray!12}0.722 $\pm$ 0.014
& \cellcolor{gray!12}0.731 $\pm$ 0.010
& \cellcolor{gray!12}0.733 $\pm$ 0.012
& \cellcolor{gray!12}0.719 $\pm$ 0.014
& \cellcolor{gray!12}0.728 $\pm$ 0.016
& \cellcolor{gray!12}0.732 $\pm$ 0.012
& \cellcolor{gray!12}\underline{0.735 $\pm$ 0.016}
& \cellcolor{gray!12}\textbf{0.752 $\pm$ 0.014} \\

\specialrule{0.55pt}{0.9pt}{0pt}
\specialrule{1.2pt}{0pt}{0pt}

\end{tabular}%
}
\end{table*}
\section{Discussion}
\subsection{\textcolor{black}{Spectral mode analysis}}
\label{sec:spectral_mode_analysis}

\textcolor{black}{
The normalised spectral energy formulation of SURE enables the analysis in terms of singular-value spectrum. If the singular values of $W \in \mathbb{R}^{M \times N}$ are ordered as $\sigma_1 \geq \cdots \geq \sigma_{\min(M,N)}$, we define truncated spectral-energy scores to understand the energy distribution,
}
\[
\textcolor{black}{
\begin{aligned}
E_{\mathrm{Top}\text{-}k}(W)
&= \sum_{r=1}^{k}\sigma_r^2, \\
U_{\mathrm{Top}\text{-}k}(W)
&= 1-\frac{E_{\mathrm{Top}\text{-}k}(W)}{MN}.
\end{aligned}
}
\]
\textcolor{black}{
Here, \textbf{Top-1} uses only the largest singular value, while \textbf{Top-2} uses the two largest singular values. Since our experimental setting uses $M=3$ and $N=10$, the rank of $W$ is at most $3$.
}

\noindent\textcolor{black}{
We find that Top-1 and Top-2 closely match the full SURE score across both SVAMP and Trivia-QA, as detailed in Appendix Table~\ref{tab:strawman_baselines}. This indicates that most of the useful cross-temperature agreement energy is concentrated in the leading singular mode.   When the model is certain, high-temperature probes align with this dominant anchor mode. When the model is uncertain, the high-temperature probes drift away from the anchor semantics, weakening this dominant mode.
This analysis clarifies the role of the spectral formulation.
}
\subsection{\textcolor{black}{Ablations}}
\label{sec:ablations}
\subsubsection{Effect of input augmentations}
\textcolor{black}{
\noindent\textbf{Non-Rephrased} uses only the original input, and the probe set consists only of high-temperature samples from the original query.\\
\noindent\textbf{Same-Input} uses semantically equivalent rephrased inputs, but does not use the mix of rephrased and low-T sampled inputs. Instead, it evaluates each rephrase independently: for each rephrased input, we generate $10$ high-temperature probes, compute the uncertainty score, and average the resulting AUROCs across rephrases.\\
\noindent\textcolor{black}{\textbf{Rephrased-only} computes uncertainty using $10$ rephrased inputs with high temperature responses obtained by sampling each of them once. The rephrase-only setting improves marginally, indicating that input rephrasing is useful, but it adds $2$X rephrasing cost. Same ablation on DSE and SNNE are added in Appendix~\ref{app:input_aug_ablation}.}\\
\noindent\textbf{SURE (Mixed-Rephrased)} uses $5$ equivalent input augmentations and $2$ high-temperature responses from each to form $10$ probe responses.
}\\
\noindent\textcolor{black}{
Figure~\ref{fig:ablation_and_baselines} shows the average AUROC across models, with full per-model results provided in Appendix~\ref{app:input_aug_ablation}. Input augmentation significantly improves performance over the non-rephrased variant. Same-Input also improves over Non-Rephrased, showing that equivalent-input augmentation itself is beneficial. SURE (Mixed-Rephrased) gives the best average performance on both datasets, suggesting that mixing probe responses exposes the most informative set of semantic variations.
}

\subsubsection{Dissecting cross-temperature agreement.}
\textcolor{black}{
We next isolate the role of cross-temperature comparison by comparing SURE with same-temperature controls.\\
\noindent\textbf{Low-T} computes spectral energy using only low-temperature samples.\\
\noindent\textbf{High-T} uses only high-temperature samples. For a matched setup, both controls use $10$ samples.\\
\noindent As shown in Figure~\ref{fig:ablation_and_baselines}, Low-T performs poorly, indicating that low-temperature generations are often too stable to reveal uncertainty. High-T is stronger because it exposes semantic variability, but it still underperforms SURE. This supports the central idea of SURE: low-temperature samples provide stable anchors, while high-temperature probes test whether this stable belief persists under stochastic sampling.
}
\begin{figure}[t!]
    \centering
    \includegraphics[width=0.88\linewidth, height=0.37\textwidth]{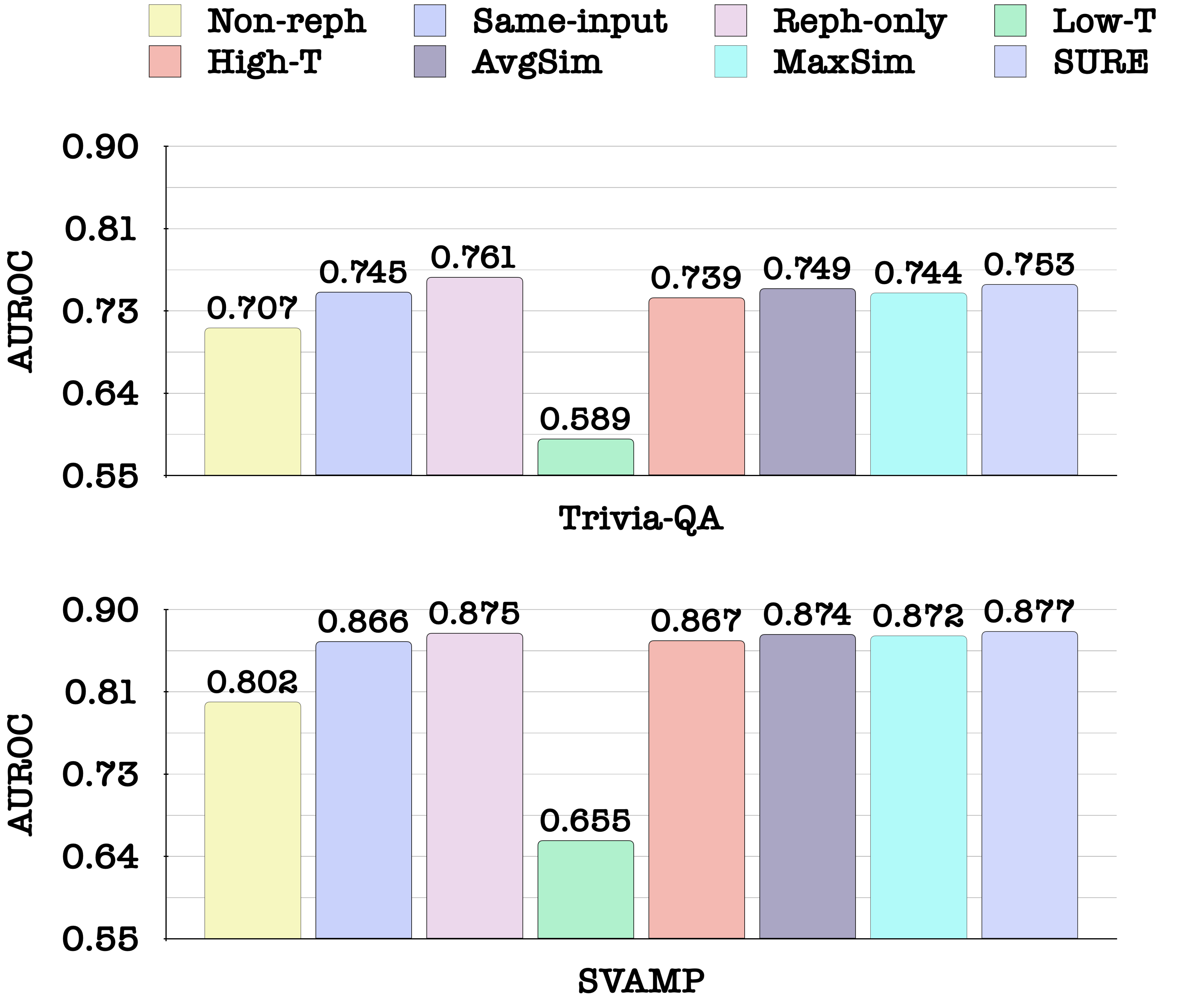}
    \caption{\textcolor{black}{\textbf{Ablations and direct controls for SURE.}
Average AUROC across the six text-QA LLMs on Trivia-QA and SVAMP for input-augmentation variants, same-temperature controls, and direct aggregation baselines. Equivalent-input augmentation and cross-temperature comparison majorly contribute to the final performance.}}
    \label{fig:ablation_and_baselines}
\end{figure}
\subsubsection{Temperature sensitivity}
\label{sec:temperature_sensitivity_main}
\textcolor{black}{
We analyze the sensitivity of SURE performance to the varying anchor temperature ($T_L$) and probe temperature ($T_H$). With $T_H=1.0$, performance remains stable across ($T_L\in\{0.1,0.2,0.3,0.4\}$) as shown in Figure~\ref{fig:temp_ablation}, confirming that low-temperature responses provide stable semantic anchors. With $T_L=0.1$, performance generally improves as $T_H$ increases toward $0.9$-$1.0$, but declines at higher temperatures. This reveals a trade-off that probes that are too stable fail to expose uncertainty, whereas excessively high temperatures produce noisy generations. We use $T_L=0.1$ and $T_H=1.0$ throughout without tuning specific to the dataset and model. Additional gridwise analysis is provided in Appendix~\ref{app:temperature_sensitivity}.
}
\begin{figure}[t!]
    \centering
    \includegraphics[width=1.0\linewidth, height=0.43\textwidth]{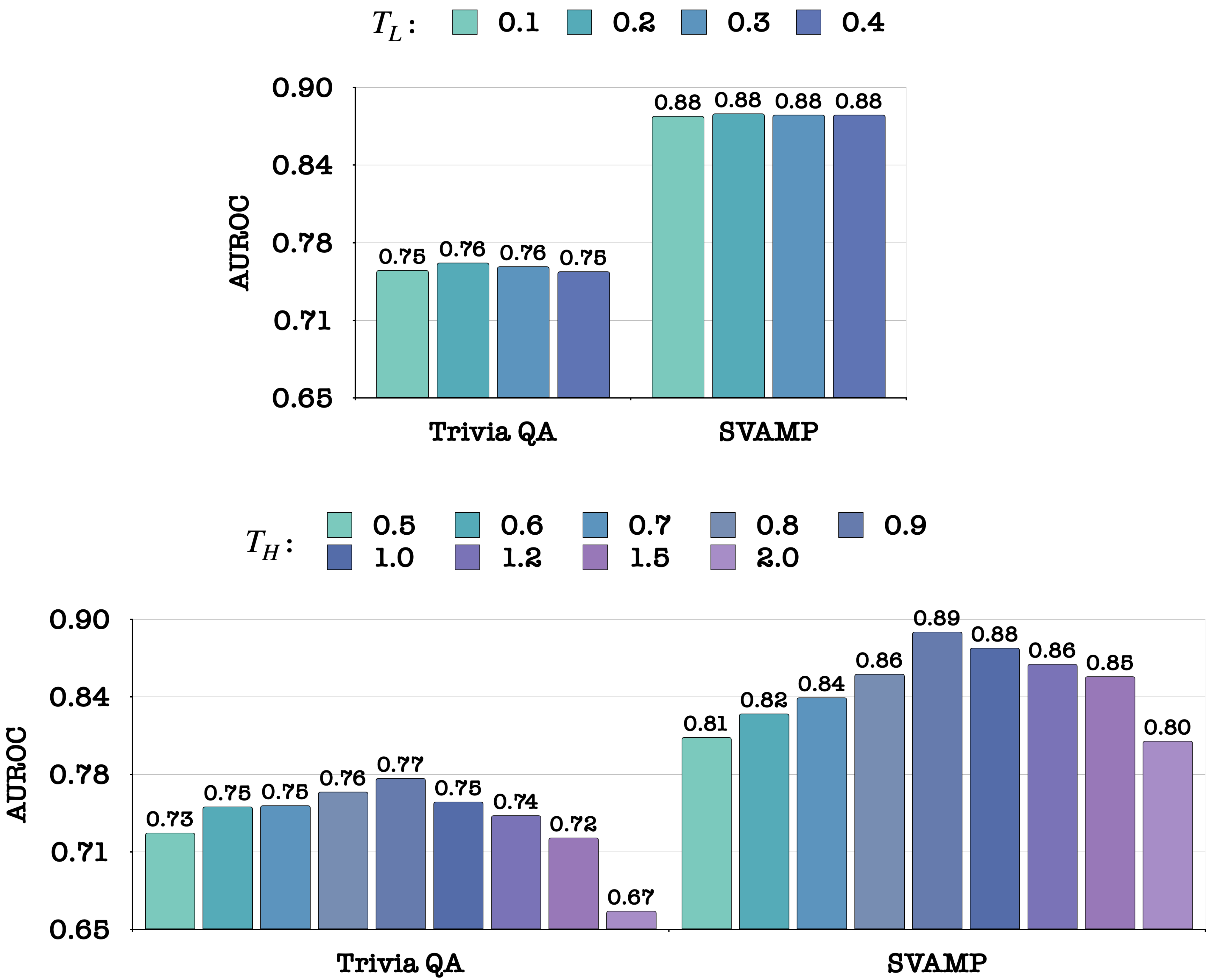}
    \caption{\textcolor{black}{\textbf{Temperature sensitivity of BiG-SURE.} \textit{Top:} Abstention AUROC when varying the anchor temperature ($T_L$), with the probe temperature fixed at ($T_H=1.0$). Performance remains stable across ($T_L\in\{0.1,0.2,0.3,0.4\}$). \textit{Bottom:} Abstention AUROC when varying ($T_H$), with ($T_L=0.1$). Performance generally improves as ($T_H$) approaches $0.9$-$1.0$, while excessively high temperatures degrade performance, revealing a trade-off.
    }}
    \vspace{-0.05in}
    \label{fig:temp_ablation}
\end{figure}
\subsection{\textcolor{black}{Strawman baselines}}
\label{sec:strawman_baselines}
\textcolor{black}{
We compare SURE against direct aggregation controls computed on the same $W$. These include average similarity (\textsc{AvgSim}), maximum anchor--probe similarity (\textsc{MaxSim}), anchor--probe entropy, and average squared similarity (\textsc{AvgSqSim}). Figure~\ref{fig:ablation_and_baselines} compares SURE with \textsc{AvgSim} and \textsc{MaxSim}. Full results are shown in Appendix~\ref{app:strawman_baselines}. \textsc{AvgSqSim} is numerically identical to BiG-SURE (Eqn.~\ref{eqn:sum_singular_values}) and serves as a sanity check.
 \textcolor{black}{This ablation study shows that the specific aggregation method used in SURE carries little value, and it is quite robust to aggregation methods like average similarity, maximum similarity and so on. The squared-similarity formulation remains useful for spectral analysis to inspect singular/spectral modes from a diagnostic perspective (Section~\ref{sec:spectral_mode_analysis}). Anchor-probe entropy performs poorly as complete anchor-probe dissimilarity ($W_{ij}=0$) indicates high uncertainty, while anchor-probe entropy approaches zero in such cases.}
}
\begin{figure}[t!]
    \centering
    \includegraphics[width=1.0\linewidth, height=0.27\textwidth]{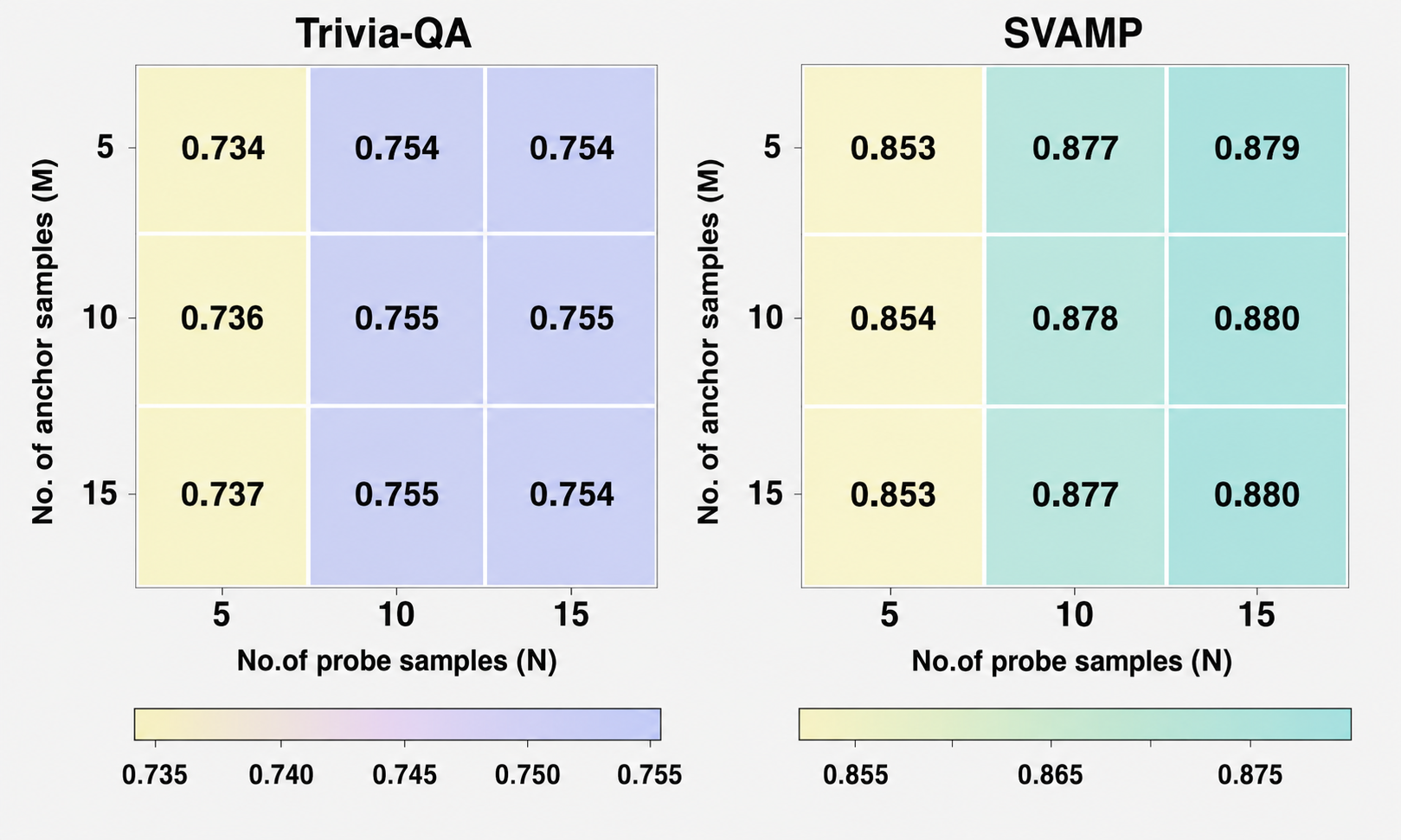}
    \caption{\textcolor{black}{\textbf{Sample sensitivity analysis for SURE.} Abstention AUROCs averaged over LLMs for different numbers of anchor samples ($M$) and probe samples ($N$) for Trivia-QA and SVAMP.
    }}
    \vspace{-0.1in}
    \label{fig:sample_sensitivity}
\end{figure}
\subsection{Sample sensitivity analysis}
\label{sec:sample_sensitivity}
\textcolor{black}{We study the sensitivity of SURE to the number of anchors ($M$) and probes ($N$) in Figure~\ref{fig:sample_sensitivity}. Across both datasets, the performance is more sensitive to the number of probe samples than to the number of anchor samples. Increasing $N$ from $5$ to $10$ gives the largest improvement. Further increasing $N$ yields only marginal gains, indicating saturation. In contrast, increasing $M$ has little effect for fixed $N$, suggesting a small number of stable low-temperature anchors is sufficient.}
\subsection{Human subjective evaluation}
\label{subsec:human-subjective-study}
\noindent \textcolor{black}{\textbf{Multilingual correctness judge:}
For multilingual SciQ, we use Gemini-2.5-Flash to obtain correctness labels by comparing the question, reference answer, and greedy prediction.
To assess the reliability of these labels, we conducted a human evaluation on the French, Japanese, and Chinese subsets of SciQ. For each language, we randomly sampled $20$ responses generated by Aya, balanced between $10$ labelled correct and $10$ labelled incorrect by Gemini-2.5-Flash. $10$ human annotators per language independently judged whether each prediction was correct with respect to the question and reference answer. We then obtain a human-consensus label through majority voting.
Gemini agreed with the human consensus on 90\% of the French examples, 95\% of the Japanese examples, and 90\% of the Chinese examples, with Cohen's $\kappa$ of $0.80$, $0.90$, and $0.80$, respectively, showing strong agreement in this limited sample study.}

\noindent \textcolor{black}{\textbf{Evaluation of rephrasing quality}
SURE assumes that the LLM-generated rephrased questions preserve the meaning of the original input. To evaluate this, we conducted a human study in English and Japanese. For each language, we constructed a controlled set of $20$ original-rephrase pairs. $10$ human annotators per language independently rated the semantic equivalence of each pair on a five-point Likert scale. Ratings from $1$ to $2$ were treated as \emph{non-equivalent}, while ratings from $3$ to $5$ were treated as \emph{equivalent}. For every pair, we obtained a binary human-consensus label through majority voting. The consensus labels agreed with the equivalence labels for all $20$ examples in both English and Japanese, yielding $100\%$ agreement and Cohen's $\kappa=1.00$ within this limited subjective study.}

\subsection{Computational complexity}
A practical consideration for black-box uncertainty estimation is the cost of uncertainty computation. For SURE, the bipartite similarity matrix \(W \in \mathbb{R}^{M \times N}\) formation leads to complexity of $\mathcal{O}(MN \, C_{\text{sem}})$, where \(C_{\text{sem}}\) denotes a single semantic similarity evaluation cost.
The subsequent computation of spectral energy is comparatively lightweight.
 Most existing methods construct a full pairwise similarity matrix over \(N\) sampled responses, leading to
$\mathcal{O}(N^2 \, C_{\text{sem}})$.
 KLE further incurs \(\mathcal{O}(N^3)\) cost. 
However, as mentioned by \cite{nikitin2024kernel},  the uncertainty computation time remains lightweight compared to the LLM inference cost. Hence, we also perform a matched-inference budget comparison to reassess the gains.\\

\subsection{Matched inference-budget comparison}
From the LLM inference cost perspective, SURE uses $M+N=3+10$ responses per input, whereas most prior sampling-based baselines use $N=10$ responses. To ensure that the gains are not simply due to more LLM calls, we re-evaluate the prior works with equalized budget of $N=13$ for text-QA tasks. As shown in Appendix~\ref{appendix:matched-inference-comparison}, SURE outperforms the baselines under this setting as well.
\section{Conclusion}

We introduced BiG-SURE, a black-box uncertainty estimation framework based on cross-temperature bipartite semantic agreement. BiG-SURE uses low-temperature responses as stable anchors and high-temperature responses as probes, constructs an anchor--probe semantic similarity graph, and defines uncertainty as the complement of its normalized squared spectral energy. Empirically, SURE achieves strong abstention AUROC across text-only, multilingual, and multimodal QA tasks, improving over prior black-box uncertainty estimators in most evaluated settings. Our ablations further show that the cross-temperature comparison and the equivalent-input augmentation majorly contribute to the final performance. These results suggest that SURE is a simple, interpretable, and practical method for black-box uncertainty estimation, making BiG-SURE a promising candidate for reliability-aware deployment of LLMs and VLMs in safety-critical applications.
\section{Limitations}

While \textsc{BiG-SURE} shows strong empirical performance across text-only, multilingual, and multimodal QA settings, it has the following limitations.\\

\noindent \textbf{Dependence on the semantic similarity backend.}
SURE relies on an entailment-based semantic similarity function to populate the anchor--probe bipartite graph. If this scorer fails to reliably capture semantic equivalence, the resulting similarity matrix may not accurately reflect the model's uncertainty. This dependence is especially important in multilingual and multimodal settings, where semantic matching across languages or visually grounded responses can itself be challenging.\\

\noindent \textcolor{black}{
\textbf{Consistency does not guarantee correctness.}
SURE can fundamentally underperform if a model mispredicts and produces the same incorrect answer semantically consistently during sampling. In this case, even a perfect semantic-similarity backend causes a consistency-based estimator to be overconfident. This failure mode can be distinguished from errors caused by the NLI backend. Further analysis and examples are discussed in Appendix~\ref{appendix:failure_analysis}.\\}

\noindent \textcolor{black}{
\textbf{Visual grounding failure.}  
For visual QA, the current semantic-similarity backend operates only on the generated textual responses and does not examine the image. Consequently, SURE measures textual generation consistency without considering image-text grounding. It may, therefore, assign high confidence when a VLM consistently produces the same visually overlooked answer. Extending the agreement function to incorporate non-text modalities is an important future work. Further analysis with examples are in Appendix~\ref{appendix:failure_analysis}.\\}

\noindent \textcolor{black}{
\textbf{Limited subjective assessment.} BiG-SURE depends on the quality of meaning-preserving paraphrases, while Gemini-2.5-Flash is used to assign correctness labels for multilingual QA. Our human evaluations provide encouraging evidence for both components, but their limited scale does not constitute comprehensive validation across languages, domains, or question types. These results should therefore be interpreted as sanity checks rather than definitive validation of either the rephrasing model or the correctness judge. A systematic human evaluation of the NLI-based semantic-similarity backend remains for future work.\\
}

\noindent \textbf{Evaluation scope.}
Our experiments focus on question-answering style benchmarks, including factual QA, math word problems, multilingual science QA, and visual QA. Although these tasks cover multiple languages, modalities, and model families, they do not fully represent broader open-ended generation settings such as long-form generation, summarization, dialogue, planning, or code synthesis. Extending cross-temperature bipartite agreement to these broader tasks remains an important direction for future work.

\section*{Acknowledgments}
We thank Google Research for its support through the Gemma Academic Award, which substantially helped meet the computational requirements of this work. We also gratefully acknowledge Google’s support for our conference travel, provided through the Kotak IISc AI–ML Centre.

\bibliography{references}
\appendix
\section{Appendix}
\subsection{Theoretical properties of SURE as an uncertainty estimator}
\label{appn:sure-valid-estimator}
\begin{enumerate}[label={},leftmargin=0pt,itemsep=1em]
    \item
    \begin{theoremstmtbox}
    \textbf{1. Semantic identity:} \textit{If all high temperature sampled prediction sequences are semantically identical and aligned with low-temperature samples, then $W$ converges to ${{\mathbbm{ 1}}}_{M\times N}$ and $\mathcal{U}_{\texttt {norm}}(W;\theta)=0$.}
    \end{theoremstmtbox}
    \noindent\textit{Intuition.}
When all high-temperature generations are semantically identical, they all lie in the same semantic cluster as the low-temperature anchors. Consequently, every anchor--probe pair exhibits maximal entailment, so $W_{ij}\approx 1$ for all $i,j$, and $W$ approaches the matrix of all ones $\mathbbm{1}_{M\times N}$. The spectral energy can be expressed as square of the Frobenius norm $\|W\|_F$ (proof in Appendix~\ref{appendix:confidence_norm}).  With $W$ being approximately equal to $\mathbbm{1}_{M \times N}$, the normalized spectral energy approaches a value of $1$ and the uncertainty, defined as its complement, approaches a value of $0$. In other words, perfect semantic consensus across temperatures leaves no room for ambiguity. The detailed proof is given in Proposition~\ref{appendix:semantic_identity}.

    \item
    \begin{theoremstmtbox}
    \textbf{2. Semantic disjointness:} \textit{If all high temperature responses are mutually semantically disjoint, and mis-aligned with low-temperature samples, then $W$ converges to $\mathbf{0}_{M\times N}$ and $\mathcal{U}_{\texttt{norm}}(W;\theta)=1$.}
    \end{theoremstmtbox}
    \noindent\textit{Intuition.}
If the high-temperature samples are mutually semantically disjoint (i.e., they do not align with the low-temperature meaning and also share no common semantic cluster), then each probe $h_j$ exhibits negligible semantic entailment with every anchor $a_i$. The similarity matrix satisfies $W_{ij}\approx 0$ for all $(i,j)$, and $W$ approaches the all-zero matrix $\mathbf{0}_{M\times N}$. In this case, the spectral energy $\|W\|_F$ goes to $0$.  As a result, $\mathcal{U}_{\texttt{norm}}(W;\theta)=1$. Intuitively, complete semantic mismatch across temperatures corresponds to maximal ambiguity which results in maximal value of SURE metric. A formal proof is shown in Proposition~\ref{appendix:semantic_disjoint}.

    \item
    \begin{theoremstmtbox}
    \textbf{3. Monotonicity:} \textit{For fixed $(M,N)$, if $W'$ corresponds to an input $x'$ where all  elements of  $W'\preceq W$, then,   $\mathcal{U}_\texttt{norm}(W';\theta) \geq  \mathcal{U}_\texttt{norm}(W;\theta)$.}
    \end{theoremstmtbox}
    \noindent\textit{Intuition.}
Holding the number of low and high temperature responses fixed, an increased  semantic diversity in the high-temperature set typically reduces the cross-temperature agreement: each high temperature response $h_j$ has reduced alignment with  low temperature response $a_i$, and the pairwise similarities shrink. The condition $W' \preceq W$ element-wise formalizes this weakening. Since the (normalized) spectral energy is proportional to the Frobenius norm, decreasing the entries decreases $\|W\|_F$, yielding $\mathcal{U}_{\texttt{norm}}(W';\theta) \geq  \mathcal{U}_{\texttt{norm}}(W;\theta)$. Intuitively, when predictions spread across more semantic modes, cross-temperature consistency breaks down and the uncertainty increases. A formal proof is given in Proposition~\ref{prop:monotonicity_unorm_overall}.
\end{enumerate}
\begin{proposition}
\label{appendix:confidence_norm}
Let \(W \in \mathbb{R}^{M\times N}\) with entries satisfying \(0 \le W_{ij} \le 1\) for all \(i,j\). Let
\[
E_{SURE}(W;\theta)=\sum_{i=1}^{r}\sigma_i^2
\]
\[
r=\operatorname{rank}(W)\le \min(M,N),
\]
where \(\sigma_i\) are the non-zero singular values of \(W\). Define
\[
E_{norm}(W;\theta)=\frac{E_{SURE}(W;\theta)}{MN}.
\]
Then
\[
0 \le E_{norm}(W;\theta)\le 1.
\]
Consequently, if
\[
U_{norm}(W;\theta)=1-E_{norm}(W;\theta),
\]
then
\[
0 \le U_{norm}(W;\theta)\le 1.
\]
\end{proposition}

\begin{proof}
Let the singular value decomposition of \(W\) be
\[
W=U\Sigma V^\top,
\]
where \(U\in\mathbb{R}^{M\times M}\) and \(V\in\mathbb{R}^{N\times N}\) are orthogonal matrices, and
\[
\Sigma=\operatorname{diag}(\sigma_1,\sigma_2,\dots,\sigma_r,0,\dots,0).
\]
Using the unitary invariance of the Frobenius norm,
\[
\|W\|_F^2 = \|U\Sigma V^\top\|_F^2 = \|\Sigma\|_F^2.
\]
Since the Frobenius norm of a diagonal matrix is the sum of squares of its diagonal entries, we obtain
\[
\|W\|_F^2=\sum_{i=1}^{r}\sigma_i^2.
\]
Therefore,
\[
E_{SURE}(W;\theta)=\|W\|_F^2.
\]

Now, by definition of the Frobenius norm,
\begin{equation}
\label{eqn:frobenius_norm}
E_{SURE}(W;\theta)=\|W\|_F^2=\sum_{i=1}^{M}\sum_{j=1}^{N} W_{ij}^2.
\end{equation}
Since \(0 \le W_{ij}\le 1\), it follows that \(0 \le W_{ij}^2\le 1\) for all \(i,j\). Hence,
\[
0 \le \sum_{i=1}^{M}\sum_{j=1}^{N} W_{ij}^2 \le MN.
\]
Thus,
\[
0 \le E_{SURE}(W;\theta)\le MN.
\]
Dividing throughout by \(MN\) yields
\[
0 \le E_{norm}(W;\theta)=\frac{E_{SURE}(W;\theta)}{MN}\le 1.
\]
Finally, since
\[
U_{norm}(W;\theta)=1-E_{norm}(W;\theta),
\]
we immediately obtain
\[
0 \le U_{norm}(W;\theta)\le 1.
\]
\end{proof}
\begin{proposition}
\label{appendix:semantic_identity}
Given $M$ low-temperature and $N$ high-temperature responses, $U_{norm}\rightarrow 0$ iff all sampling responses are semantically similar.
\end{proposition}
\begin{proof}
To prove this, we show the following:
\begin{itemize}
    \item $U_{norm}\rightarrow 0$ if all sampling responses are semantically similar.
    \item All sampling responses are semantically similar if $U_{norm}\rightarrow 0$.
\end{itemize}
\noindent \textbf{Case I}: As all sampling responses are semantically similar, let $w_{ij}=s, \forall i,j$.
Then the bi-adjacency matrix is constant:
\[
W = s\,\mathbf{1}_{M\times N},
\qquad w_{ij}=s,\ \forall i,j,
\]
where $\mathbf{1}_{M\times N}$ is the all-ones matrix.
The Frobenius norm satisfies
\[
\|W\|_F^2 = \sum_{i=1}^{M}\sum_{j=1}^{N} w_{ij}^2
= \sum_{i=1}^{M}\sum_{j=1}^{N} s^2
= MN\,s^2.
\]
Since $E_{SURE}(W;\theta)=\sum_{i=1}^{\min\{M,N\}}\sigma_i^2=\|W\|_F^2$ (Eqn.~\ref{eqn:frobenius_norm}), we obtain
\[
E_{SURE}(W;\theta)=MN\,s^2.
\]
Therefore the normalized confidence is
\[
E_{\mathrm{norm}}(W;\theta)
=\frac{E_{SURE}(W;\theta)}{MN}
=\frac{MN\,s^2}{MN}
=s^2,
\]
and the normalized uncertainty is
\[
U_{\mathrm{norm}}(W;\theta)=1-E_{\mathrm{norm}}(W;\theta)=1-s^2.
\]
Hence, as $s\to 1$ (perfect semantic similarity), we have $U_{\mathrm{norm}}(W;\theta)\to 0$.\\

\noindent \textbf{Case II}: Suppose $U_{\mathrm{norm}}(W;\theta)=0$. Then $E_{\mathrm{norm}}(W;\theta)=1$, hence
\[
E_{SURE}(W;\theta)=MN.
\]
Therefore
\[
\sum_{i,j} W_{ij}^2 = MN.
\]
Since $0\le W_{ij}\le 1$, we have $W_{ij}^2\le 1$ for each entry, and thus
\[
\sum_{i,j} W_{ij}^2 \le \sum_{i,j} 1 = MN,
\]
with equality if and only if $W_{ij}^2=1$ for every $(i,j)$, i.e., $W_{ij}=1$ for all $(i,j)$. Hence $W=\mathbf{1}_{M\times N}$.\\


\noindent Finally, if $W_{ij}=s$ for all $i,j$, then $W=\mathbf{1}_{M\times N}$ implies $s=1$, and conversely $s=1$ implies $W=\mathbf{1}_{M\times N}$, completing the iff statement.
\end{proof}
\begin{proposition}
\label{appendix:semantic_disjoint}
Given $M$ low-temperature and $N$ high-temperature responses, $U_{\mathrm{norm}}(W;\theta)=1$ iff all cross-temperature response pairs are semantically disjoint, i.e., $\mathrm{sim}(a_i,h_j)=0$ for all $i,j$.
\end{proposition}

\begin{proof} We show both directions of implications to prove the iff condition.\\
\noindent \textbf{Case I}: If $U_{\mathrm{norm}}(W;\theta)=1$, then $E_{\mathrm{norm}}(W;\theta)=0$, hence $E_{SURE}(W;\theta)=0$ and therefore
\[
\sum_{i,j} W_{ij}^2 = 0.
\]
Since each term $W_{ij}^2\ge 0$, the only way their sum is zero is if $W_{ij}^2=0$ for every $(i,j)$, i.e., $W_{ij}=0$ for all $(i,j)$. Thus $W=\mathbf{0}_{M\times N}$.

\noindent \textbf{Case II}: Conversely, if $W=\mathbf{0}_{M\times N}$, then $\|W\|_F^2=0$, hence $E_{\mathrm{norm}}(W;\theta)=0$ and $U_{\mathrm{norm}}(W;\theta)=1$.
\end{proof}
\begin{theorem}
\label{appendix:cluster_calc}
For $M$ low and $N$ high-temperature responses, we show that,
\begin{enumerate}
    \item The number of feasible clusters $K(W)$ increases with decreasing $W_{ij}$, and vice-versa.
    \item $\mathbb{P}\big(C_N(W)=N\big) = \frac{(K(W))_N}{(K(W))^N}$
    \item $\mathbb{P}\big(C_N(W)<N\big) \leq \binom{N}{2}\frac{1}{K(W)}$
    \item $\mathbb{E}[C_N(W)] = \sum_{i=1}^{K(W)}1-(1-p_i)^{N}\leq N$
    \item Under uniform cluster probability distribution and $0\leq W_{ij}< 1$, $\mathbb{E}[C_N(W)]\rightarrow N$.
    \item Also, $\mathbb{E}[C_N(W)]$ is a monotonic function of $K(W)$.
\end{enumerate}
\end{theorem}
\begin{proof} For simplicity, we assume that the low-temperature responses are semantically similar, and the high-temperature responses show semantic variability.\\
\noindent Let $N$ be the number of high-temperature samples. Also, given the similarity matrix $W$, let the number of possible semantic clusters is $K(W)$ (for example, for $W=\mathbf{1}_{M\times N}$, $K(W)=1$). But $K(W)$ clusters are observed only if an infinite number of samples are drawn. Because of the finite sampling of $N$, we do observe a subset of them. 

    Assume the $M$ low-temperature responses $\{q_i\}_{i=1}^M$ be semantically consistent and belong to a single semantic cluster, denoted by index $\ell$. Let $\{1,\dots,\bar{K}\}$ be a universal set of semantic clusters for a fixed input.
Associate each cluster $j$ with a prototype similarity to the cluster $\ell$,
\[
\mu_j \triangleq \mathrm{sim}(j,\ell)\in[0,1],
\qquad \mu_\ell = 1,
\]
where larger $\mu_j$ means cluster $j$ is semantically closer to the cluster $\ell$.\\
\noindent We assume that, conditional on a high-temperature response belonging to cluster $j$, the observed cross-temperature similarity score $s\in[0,1]$ follows a $\beta$-distribution and has a likelihood of
\[
L_j(s)\ \propto\ s^{\nu \mu_j - 1}(1-s)^{\nu(1-\mu_j)-1},
\qquad \nu>0,
\]
i.e., $s\mid(Z=j)\sim\mathrm{Beta}(\nu\mu_j,\nu(1-\mu_j))$.
Based on a threshold $\alpha\in(0,1]$ and define the set of feasible clusters at similarity level $s$ by a likelihood-ratio rule:
\[
\mathcal{F}_\alpha(s)\ \triangleq\ \left\{j\in\{1,\dots,\bar{K}\}:\ \frac{L_j(s)}{L_\ell(s)}\ \ge\ \alpha\right\}.
\]
Define the number of feasible clusters as $K(W) = K(s)\triangleq|\mathcal{F}_\alpha(s)|$.
Using the Beta likelihood above and $\mu_\ell=1$, the log-likelihood ratio simplifies (up to a constant independent of $j$) to
\[
\log\frac{L_j(s)}{L_\ell(s)}
=
\nu(\mu_j-1)\log s
\]
\[
+ \nu(1-\mu_j)\log(1-s) + \text{const}(j)
\]
For $s$ bounded away from $1$ and in the low-similarity regime (small $s$), the dominant term is
\[
\nu(\mu_j-1)\log s,
\]
since $\log s\to -\infty$ as $s\downarrow 0$ while $\log(1-s)=O(1)$.
Because $\mu_j<1$ for $j\neq \ell$, we have $(\mu_j-1)<0$, hence $(\mu_j-1)\log s$ increases as $s$ decreases (it becomes large positive).
Therefore, for each fixed $j\neq \ell$, the likelihood ratio $L_j(s)/L_\ell(s)$ is nondecreasing as $s$ decreases, implying that once a cluster satisfies the feasibility condition at some $s_2$, it will also satisfy it for any smaller $s_1<s_2$.
Thus $\mathcal{F}_\alpha(s_1)\supseteq \mathcal{F}_\alpha(s_2)$ and $K(s_1)\ge K(s_2)$ ($K(W_1)\ge K(W_2)$).\\

\noindent Let, $C_N(W)$ be the observed number of semantic clusters among $K(W)$ because we use finite $N$ samples instead of infinite, and $C_N(W) \leq K(W)$. Assuming $N$ is not too large and $K(W)\geq N$,
\begin{equation}
\begin{aligned}
\mathbb{P}\big(C_N(W)=N\big)
&= \frac{K(W)}{K(W)} \cdot \frac{K(W)-1}{K(W)}\cdots \\
&\cdots \frac{K(W)-N+1}{K(W)} \\
&= \frac{(K(W))_N}{(K(W))^N}
\end{aligned}
\end{equation}
Also, let the event of responses $i$ and $j$ being assigned to the same cluster be $E_{ij}$. If $Z_{i}$ is the semantic cluster id assigned to response $i$, then
\begin{equation}
    E_{ij} = \{Z_i= Z_j\},\, \mathbb{P}(E_{ij}) = \frac{1}{K(W)}
\end{equation}
Now,
\begin{equation}
\begin{aligned}
    \mathbb{P}(\exists i<j: Z_i=Z_j)&=\mathbb{P}(\cup_{i<j}E_{ij}) \\
    \leq \sum_{i<j}\frac{1}{K(W)} \\
    \implies \mathbb{P}\big(C_N(W)<N\big) &\leq \binom{N}{2}\frac{1}{K(W)}
\end{aligned}
\end{equation}
\\
Let the cluster probabilities among the $K(W)$ feasible set of clusters given $W$ are $\{p_i\}_{i=1}^{K(W)}$. We calculate the expected number of observed clusters among $N$ samples.
\begin{equation}
\begin{aligned}
    C_N(W) &= \sum_{i=1}^{K(W)}I_i\\
    I_i=\{\text{cluster}\, i\,& \text{appears at least once among}\, N\}
\end{aligned}
\end{equation}
\begin{equation}
\begin{aligned}
    \mathbb{E}[C_N(W)] &= \mathbb{E}\bigg[\sum_{i=1}^{K(W)}I_i\bigg]\\
    &= \sum_{i=1}^{K(W)}\mathbb{E}[I_i]\\
    &=\sum_{i=1}^{K(W)}1-(1-p_i)^{N}
\end{aligned}
\end{equation}
\textbf{\underline{Case I}}: For $W_{ij}=1$, implies $K(W)=1$ and hence, $p_i=1$. So,
\begin{equation}
    \mathbb{E}[C_N(W)] = 1-(1-1)^N = 1
\end{equation}
\textbf{\underline{Case II}}: For $W_{ij}=0$, implies $K(W)$ is large. So,
\begin{equation}
\begin{aligned}
    \mathbb{E}[C_N(W)] &= \sum_{i=1}^{K(W)}1-(1-p_i)^N\\
    = &\sum_{i=1}^{K(W)}\bigg[1-\big(1-Np_i+\binom{N}{2}p_i^2-...\big)\bigg]\\
    = &\sum_{i=1}^{K(W)}\bigg[Np_i-\binom{N}{2}p_i^2+...\bigg]\\
    = & \,N\sum_{i=1}^{K(W)}p_i-\binom{N}{2}\sum_{i=1}^{K(W)}p_i^2+...\\
    = & \,N-\binom{N}{2}\sum_{i=1}^{K(W)}p_i^2+...\\
    \leq &\,N
\end{aligned}
\end{equation}
For uniform cluster distribution, $p_i=\frac{1}{K(W)}$, and $\mathbb{E}[C_N(W)]\rightarrow N$ as $K(W)$ is large and hence $\frac{N^2}{K(W)}<<1$.\\

\noindent For uniform distribution with $K(W)=k$ clusters,
\begin{equation}
    p_i = \frac{1}{K(W)} = \frac{1}{k}, \,k>0
\end{equation}
Now, replacing $p_i$ in the cluster equation, we get,
\begin{equation}
    \mathbb{E}[C_N] = k\bigg(1-\bigg(1-\frac{1}{k}\bigg)^N\bigg)
\end{equation}
$C_N$ is shorthand notation of $C_N(W)$. Now, we need to show, $\frac{\partial \mathbb{E}[C_N]}{\partial k} \geq 0$.\\
Denote,
\begin{equation}
    a(k) = 1-\frac{1}{k}, \,a'(k) = \frac{1}{k^2}
\end{equation}
\begin{equation}
    \begin{aligned}
        f(k) &= k - ka(k)^N\\
        \implies f'(k) &= 1 - a(k)^N -kNa(k)^{N-1}a'(k)\\
        &= 1 - a(k)^N -\frac{kN}{k^2}a(k)^{N-1}
    \end{aligned}
\end{equation}
Now, let, $x=\frac{1}{k}\in (0,1]$. So, $a = 1-x$. Now,
\begin{equation}
    \begin{aligned}
        f'(k) &= 1-(1-x)^{N-1}(1-x+Nx)\\
        &= 1-(1-x)^{N-1}(1+(N-1)x)
    \end{aligned}
\end{equation}
Now, we have to show, $(1-x)^{N-1}(1+(N-1)x)\leq 1$ for $x\in [0,1]$.
Let $m=N-1\geq 0$. Define,
\begin{equation}
    h(x) := \ln (1+mx) + m\ln(1-x)
\end{equation}
Notice that,
\begin{equation}
    \begin{aligned}
        h(x) &\leq 0\\
        \implies \exp (h(x)) &\leq 1\\
        \implies (1+mx)(1-x)^m &\leq 1
    \end{aligned}
\end{equation}
Now,
\begin{equation}
    \begin{aligned}
        h'(x) &= \frac{m}{1+mx}+m\frac{-1}{1-x}\\
        &= m\bigg(\frac{1}{1+mx} - \frac{1}{1-x}\bigg)\\
        &= -\frac{m(m+1)x}{(1+mx)(1-x)} \leq 0
    \end{aligned}
\end{equation}
Also, we know $h(0) = 0$. Hence, $h(x)\leq 0$, and hence proved.
\end{proof}
\begin{proposition}
\label{appendix:monotonicity_unorm}
For the normalized uncertainty,
\[
U_{norm}(W;\theta)=1-\frac{E_{SURE}(W;\theta)}{MN}.
\]
Then \(U_{norm}(W;\theta)\) is monotone non-increasing in each entry of \(W\). In particular, if for some fixed \((p,q)\),
\[
\widetilde W = W + \delta\, e_p e_q^\top,
\qquad \delta \ge 0,
\]
then
\[
U_{norm}(\widetilde W;\theta)\le U_{norm}(W;\theta).
\]
\end{proposition}

\begin{proof}
We have
\[
U_{norm}(W;\theta)
=
1-\frac{1}{MN}\sum_{i=1}^{M}\sum_{j=1}^{N} W_{ij}^2.
\]
Then
\[
U_{norm}(\widetilde W;\theta)
=
1-\frac{1}{MN}.
\]
\[
\left(
\sum_{(i,j)\neq (p,q)} W_{ij}^2 + (W_{pq}+\delta)^2
\right)
\]
Therefore,
\begin{align}
U_{norm}(\widetilde W;\theta)&-U_{norm}(W;\theta)\\
&=
-\frac{1}{MN}\Big((W_{pq}+\delta)^2-W_{pq}^2\Big) \\
&=
-\frac{1}{MN}\Big(2\delta W_{pq}+\delta^2\Big).
\end{align}
Since \(W_{pq}\ge 0\) and \(\delta\ge 0\), it follows that
\[
2\delta W_{pq}+\delta^2 \ge 0.
\]
Hence,
\[
U_{norm}(\widetilde W;\theta)\le U_{norm}(W;\theta).
\]
Thus, \(U_{norm}(W;\theta)\) is monotone non-increasing in each entry of \(W\).
\end{proof}

\begin{proposition}
\label{prop:monotonicity_unorm_overall}
Assume that the low-temperature responses are semantically similar. Then increasing the number of semantically unique high-temperature answers
weakens cross-temperature agreement in the sense that for two prediction sets
$\mathcal{V}$ and $\mathcal{V}'$ (with the same $M,N$),
the corresponding similarity matrices satisfy
\[
W' \preceq W \quad \text{elementwise, with } W'\neq W.
\tag{A1}
\]
Consequently, \emph{increasing semantic uniqueness of high-temperature answers monotonically increases}
$U_{\mathrm{norm}}(W;\theta)$ and simultaneously increases the expected number of observed semantic clusters.
\end{proposition}

\begin{proof}
We prove the two monotone conclusions.

\noindent\textbf{(i) $U_{\mathrm{norm}}$ increases under weakened cross-temperature agreement.}
By Proposition~\ref{appendix:monotonicity_unorm}, the map
$W\mapsto U_{\mathrm{norm}}(W;\theta)$ is monotone non-increasing in each entry of $W$.
Hence, if $W'\preceq W$ elementwise, then
\[
U_{\mathrm{norm}}(W';\theta)\ \ge\ U_{\mathrm{norm}}(W;\theta).
\]
Moreover, since $W'\neq W$, at least one entry decreases strictly, implying the inequality is strict:
\[
U_{\mathrm{norm}}(W';\theta)\ >\ U_{\mathrm{norm}}(W;\theta).
\]

\noindent\textbf{(ii) The expected number of observed semantic clusters increases.}
Since the low-temperature responses are assumed semantically similar, they define a single semantic cluster.
By Theorem~\ref{appendix:cluster_calc}(1), decreasing cross-temperature similarities (i.e., $W'\preceq W$) cannot reduce the
number of feasible semantic clusters, hence
\[
K(W')\ \ge\ K(W).
\]
By Theorem~\ref{appendix:cluster_calc}(6), $\mathbb{E}[C_N(W)]$ is a monotone nondecreasing function of $K(W)$
, and therefore
\[
\mathbb{E}[C_N(W')]\ \ge\ \mathbb{E}[C_N(W)]
\]
In particular, under uniform cluster probabilities (Theorem~\ref{appendix:cluster_calc}(5)), increasing $K(W)$ drives
$\mathbb{E}[C_N(W)]$ upward towards $N$.

\noindent
Combining (i) and (ii), the elementwise weakening of cross-temperature agreement induced by increasing
semantic uniqueness of the high-temperature answers  monotonically increases
$U_{\mathrm{norm}}(W;\theta)$ and simultaneously increases the expected number of observed semantic clusters. This supports the interpretation of $\mathcal{U}_{\texttt{norm}}(W;\theta)$ as a cross-temperature semantic instability score: as high-temperature responses spread across more semantic clusters, the uncertainty score increases- obeying the fundamental notion of uncertainty.
\end{proof}

\subsection{\textcolor{black}{Full strawman-baseline analysis}}
\label{app:strawman_baselines}

\textcolor{black}{
Table~\ref{tab:strawman_baselines} reports the full per-model comparison between BiG-SURE and direct aggregation baselines computed from the same cross-temperature anchor--probe similarity matrix $W$. This analysis is designed to isolate the contribution of the aggregation rule, since all methods use the same generated samples and the same entailment-based similarity scores.
}

\textcolor{black}{
We compare the following variants:
}
\begin{itemize}
    \item \textcolor{black}{\textbf{BiG-SURE:} the proposed normalized squared spectral-energy score, computed using all singular values of $W$.}
    \item \textcolor{black}{\textbf{AvgSim:} the mean anchor--probe similarity, $\frac{1}{MN}\sum_{i,j}W_{ij}$.}
    \item \textcolor{black}{\textbf{AvgSqSim:} the mean squared anchor--probe similarity, $\frac{1}{MN}\sum_{i,j}W_{ij}^{2}$. Since $\sum_r\sigma_r^2=\|W\|_F^2=\sum_{i,j}W_{ij}^2$, this is numerically identical to BiG-SURE.}
    \item \textcolor{black}{\textbf{MaxSim:} the maximum anchor--probe similarity score.}
    \item \textcolor{black}{\textbf{Anchor-probe Ent.:} an entropy-style statistic computed from the anchor--probe similarity distribution.}
\end{itemize}

\textcolor{black}{
The results show that BiG-SURE and AvgSqSim match exactly, as expected from the Frobenius-energy identity.
The direct controls are weaker on average. AvgSim and MaxSim underperform BiG-SURE on both datasets, while anchor--probe entropy performs substantially worse. These results support the use of squared cross-temperature agreement as a compact reliability signal.
}
\begin{table*}[t]
\centering
\caption{Comparison with direct aggregation baselines and truncated spectral variants using entailment-based similarity. Results are reported as mean $\pm$ standard deviation over 5 seeds. AvgSqSim is numerically identical to BiG-SURE because normalized spectral energy equals the normalized squared Frobenius norm of the anchor--probe similarity matrix.}
\label{tab:strawman_baselines}
\setlength{\tabcolsep}{4pt}
\small

\setlength{\aboverulesep}{0pt}
\setlength{\belowrulesep}{0pt}
\setlength{\heavyrulewidth}{1.2pt}
\setlength{\lightrulewidth}{0.55pt}

\resizebox{0.98\linewidth}{!}{%
\renewcommand{\arraystretch}{1.2}
\begin{tabular}{lc:cc:cccc}

\specialrule{1.2pt}{0pt}{1.2pt}
\specialrule{0.55pt}{0pt}{0.9pt}

\textbf{Model} & \textbf{BiG-SURE} & \textbf{Top-1} & \textbf{Top-2} & \textbf{AvgSim} & \textbf{AvgSqSim} & \textbf{MaxSim} & \textbf{Anchor-probe Ent.} \\

\midrule[0.8pt]

\multicolumn{8}{c}{\textbf{SVAMP}} \\
\midrule[0.6pt]

Llama-8B   & \textbf{0.8897 $\pm$ 0.0261} & 0.8812 $\pm$ 0.0256 & 0.8891 $\pm$ 0.0261 & 0.8856 $\pm$ 0.0329 & \textbf{0.8897 $\pm$ 0.0261} & 0.8830 $\pm$ 0.0321 & 0.4455 $\pm$ 0.0111 \\
Qwen-7B    & \textbf{0.8039 $\pm$ 0.0167} & 0.8022 $\pm$ 0.0134 & 0.8031 $\pm$ 0.0161 & 0.8037 $\pm$ 0.0185 & \textbf{0.8039 $\pm$ 0.0167} & 0.7996 $\pm$ 0.0223 & 0.4699 $\pm$ 0.0214 \\
Llama-70B  & 0.9051 $\pm$ 0.0132 & \textbf{0.9053 $\pm$ 0.0145} & 0.9051 $\pm$ 0.0132 & 0.9026 $\pm$ 0.0105 & 0.9051 $\pm$ 0.0132 & 0.9037 $\pm$ 0.0084 & 0.4899 $\pm$ 0.0149 \\
Qwen-72B   & \textbf{0.9018 $\pm$ 0.0392} & 0.9017 $\pm$ 0.0325 & 0.9018 $\pm$ 0.0380 & 0.9011 $\pm$ 0.0356 & \textbf{0.9018 $\pm$ 0.0392} & 0.9002 $\pm$ 0.0342 & 0.3714 $\pm$ 0.0124 \\
Gemma-12B  & 0.8821 $\pm$ 0.0208 & 0.8786 $\pm$ 0.0217 & 0.8820 $\pm$ 0.0208 & \textbf{0.8832 $\pm$ 0.0180} & 0.8821 $\pm$ 0.0208 & 0.8815 $\pm$ 0.0187 & 0.4414 $\pm$ 0.0131 \\
Gemma-27B  & 0.8786 $\pm$ 0.0168 & \textbf{0.8788 $\pm$ 0.0156} & 0.8785 $\pm$ 0.0162 & 0.8661 $\pm$ 0.0237 & 0.8786 $\pm$ 0.0168 & 0.8662 $\pm$ 0.0236 & 0.4766 $\pm$ 0.0151 \\

\midrule[0.4pt]
\textbf{Average} & \textbf{0.8769 $\pm$ 0.0221} & 0.8746 $\pm$ 0.0205 & 0.8766 $\pm$ 0.0217 & 0.8737 $\pm$ 0.0232 & \textbf{0.8769 $\pm$ 0.0221} & 0.8724 $\pm$ 0.0232 & 0.4491 $\pm$ 0.0147 \\

\midrule[0.8pt]

\multicolumn{8}{c}{\textbf{Trivia-QA}} \\
\midrule[0.6pt]

Llama-8B   & \textbf{0.7720 $\pm$ 0.0212} & 0.7716 $\pm$ 0.0185 & 0.7720 $\pm$ 0.0212 & 0.7624 $\pm$ 0.0237 & \textbf{0.7720 $\pm$ 0.0212} & 0.7586 $\pm$ 0.0228 & 0.4085 $\pm$ 0.0116 \\
Qwen-7B    & \textbf{0.8060 $\pm$ 0.0141} & 0.8052 $\pm$ 0.0162 & 0.8059 $\pm$ 0.0138 & 0.8059 $\pm$ 0.0153 & \textbf{0.8060 $\pm$ 0.0141} & 0.8026 $\pm$ 0.0170 & 0.4390 $\pm$ 0.0055 \\
Llama-70B  & 0.7145 $\pm$ 0.0087 & 0.7139 $\pm$ 0.0095 & 0.7144 $\pm$ 0.0060 & \textbf{0.7188 $\pm$ 0.0179} & 0.7145 $\pm$ 0.0087 & 0.7156 $\pm$ 0.0191 & 0.4731 $\pm$ 0.0111 \\
Qwen-72B   & 0.7397 $\pm$ 0.0295 & \textbf{0.7398 $\pm$ 0.0188} & 0.7397 $\pm$ 0.0295 & 0.7320 $\pm$ 0.0344 & 0.7397 $\pm$ 0.0295 & 0.7248 $\pm$ 0.0345 & 0.4124 $\pm$ 0.0183 \\
Gemma-12B  & \textbf{0.7886 $\pm$ 0.0110} & 0.7865 $\pm$ 0.0244 & 0.7883 $\pm$ 0.0162 & 0.7871 $\pm$ 0.0070 & \textbf{0.7886 $\pm$ 0.0110} & 0.7868 $\pm$ 0.0051 & 0.4974 $\pm$ 0.0113 \\
Gemma-27B  & \textbf{0.6991 $\pm$ 0.0165} & 0.6989 $\pm$ 0.0149 & 0.6991 $\pm$ 0.0165 & 0.6872 $\pm$ 0.0208 & \textbf{0.6991 $\pm$ 0.0165} & 0.6773 $\pm$ 0.0188 & 0.5121 $\pm$ 0.0311 \\

\midrule[0.4pt]
\textbf{Average} & \textbf{0.7533 $\pm$ 0.0168} & 0.7527 $\pm$ 0.0170 & 0.7532 $\pm$ 0.0172 & 0.7489 $\pm$ 0.0198 & \textbf{0.7533 $\pm$ 0.0168} & 0.7443 $\pm$ 0.0196 & 0.4571 $\pm$ 0.0148 \\

\specialrule{0.55pt}{0.9pt}{0pt}
\specialrule{1.2pt}{0pt}{0pt}

\end{tabular}%
}
\end{table*}
\subsection{\textcolor{black}{Input-augmentation ablations}}
\label{app:input_aug_ablation}

\textcolor{black}{
We provide the full per-model ablation of input variation strategies in Table~\ref{tab:ablation_rephrased_sameinput_nonreph}. The low-temperature anchors are sampled from the original input in all variants. The variants differ only in how the high-temperature probe set is constructed.
}

\begin{itemize}
    \item \textcolor{black}{\textbf{Non-Rephrased:} probes are generated only from the original input using high-temperature sampling. No equivalent-input augmentation is used.}
    \item \textcolor{black}{\textbf{Rephrased-only:} In this case, $10$ input rephrasings are used instead of $5$, and for each of them high-temperature sampling is used once.}
    \item \textcolor{black}{\textbf{Same-Input:} each equivalent rephrased input is treated independently. For each rephrase, we generate 10 high-temperature probes, compute the SURE uncertainty score, evaluate AUROC, and then average the AUROCs across rephrases.}
    \item \textcolor{black}{\textbf{SURE (Mixed-Rephrased):} this is the default setting. We use 5 equivalent input augmentations and sample 2 high-temperature responses from each, forming one mixed probe pool of 10 responses.}
\end{itemize}

\noindent\textcolor{black}{ The Figure~\ref{fig:ablation-rephrase-only} also shows that using only rephrasings improve the performance across different methods as well. However, it comes with additional rephrasing cost.
The results show that input augmentation is the main contributor: both Same-Input and Mixed-Rephrased substantially outperform Non-Rephrased on average. Mixed-Rephrased further improves over Same-Input on both dataset averages, suggesting that combining semantically equivalent perturbations within the same probe pool provides a modest but consistent additional benefit.
}
\begin{table}[t]
\centering
\caption{Ablation of input variation strategies for SURE. We compare \textit{Non-Rephrased} (no input augmentation), \textit{Same-Input} (multiple generations from the same input), and \textit{Mixed (Rephrased)} (our default setting, using semantically equivalent rephrased inputs). Results are reported as abstention AUROC ($\text{mean} \pm \text{std}$ over 5 seeds). The best result in each row is shown in \textbf{bold}, and the second-best result is \underline{underlined}.}
\label{tab:ablation_rephrased_sameinput_nonreph}
\setlength{\tabcolsep}{5pt}
\small

\setlength{\aboverulesep}{0pt}
\setlength{\belowrulesep}{0pt}
\setlength{\heavyrulewidth}{1.2pt}
\setlength{\lightrulewidth}{0.55pt}

\resizebox{0.9\linewidth}{!}{%
\renewcommand{\arraystretch}{1.22}
\begin{tabular}{l c c c}

\specialrule{1.2pt}{0pt}{1.2pt}
\specialrule{0.55pt}{0pt}{0.9pt}

\multirow{2}{*}{\textbf{Model}} &
\multicolumn{3}{c}{\textbf{AUROC ($\uparrow$)}} \\
\cmidrule(lr){2-4}
& \textbf{Non-Rephrased} & \textbf{Same-Input} & \textbf{SURE (Mixed-Rephrased)} \\

\midrule[0.8pt]

\multicolumn{4}{c}{\textbf{SVAMP}} \\
\midrule[0.6pt]

Gemma-12B & 0.7771 $\pm$ 0.0092 & \underline{0.8536 $\pm$ 0.0229} & \textbf{0.8821 $\pm$ 0.0208} \\
Gemma-27B & 0.7761 $\pm$ 0.0087 & \underline{0.8470 $\pm$ 0.0242} & \textbf{0.8786 $\pm$ 0.0168} \\
Llama-8B  & 0.7050 $\pm$ 0.0394 & \textbf{0.8989 $\pm$ 0.0239} & \underline{0.8897 $\pm$ 0.0261} \\
Llama-70B & 0.8893 $\pm$ 0.0074 & \underline{0.9031 $\pm$ 0.0118} & \textbf{0.9051 $\pm$ 0.0132} \\
Qwen-7B   & 0.7619 $\pm$ 0.0200 & \underline{0.7858 $\pm$ 0.0100} & \textbf{0.8039 $\pm$ 0.0167} \\
Qwen-72B  & 0.9011 $\pm$ 0.0110 & \textbf{0.9102 $\pm$ 0.0211} & \underline{0.9018 $\pm$ 0.0392} \\

\midrule
\textbf{Avg.} & 0.8018 & \underline{0.8664} & \textbf{0.8769} \\

\midrule[0.8pt]

\multicolumn{4}{c}{\textbf{Trivia-QA}} \\
\midrule[0.6pt]

Gemma-12B & 0.6978 $\pm$ 0.0117 & \underline{0.7606 $\pm$ 0.0113} & \textbf{0.7886 $\pm$ 0.0110} \\
Gemma-27B & 0.6051 $\pm$ 0.0086 & \underline{0.6896 $\pm$ 0.0166} & \textbf{0.6991 $\pm$ 0.0165} \\
Llama-8B  & 0.7435 $\pm$ 0.0068 & \textbf{0.7804 $\pm$ 0.0130} & \underline{0.7720 $\pm$ 0.0212} \\
Llama-70B & 0.7065 $\pm$ 0.0081 & \textbf{0.7244 $\pm$ 0.0142} & \underline{0.7145 $\pm$ 0.0087} \\
Qwen-7B   & 0.7807 $\pm$ 0.0054 & \underline{0.7885 $\pm$ 0.0087} & \textbf{0.8060 $\pm$ 0.0141} \\
Qwen-72B  & 0.7096 $\pm$ 0.0104 & \underline{0.7284 $\pm$ 0.0174} & \textbf{0.7397 $\pm$ 0.0295} \\

\midrule
\textbf{Avg.} & 0.7072 & \underline{0.7453} & \textbf{0.7533} \\

\specialrule{0.55pt}{0.9pt}{0pt}
\specialrule{1.2pt}{0pt}{0pt}

\end{tabular}%
}
\end{table}
\begin{figure}[t!]
    \centering
    \includegraphics[width=0.85\linewidth, height=0.4\textwidth]{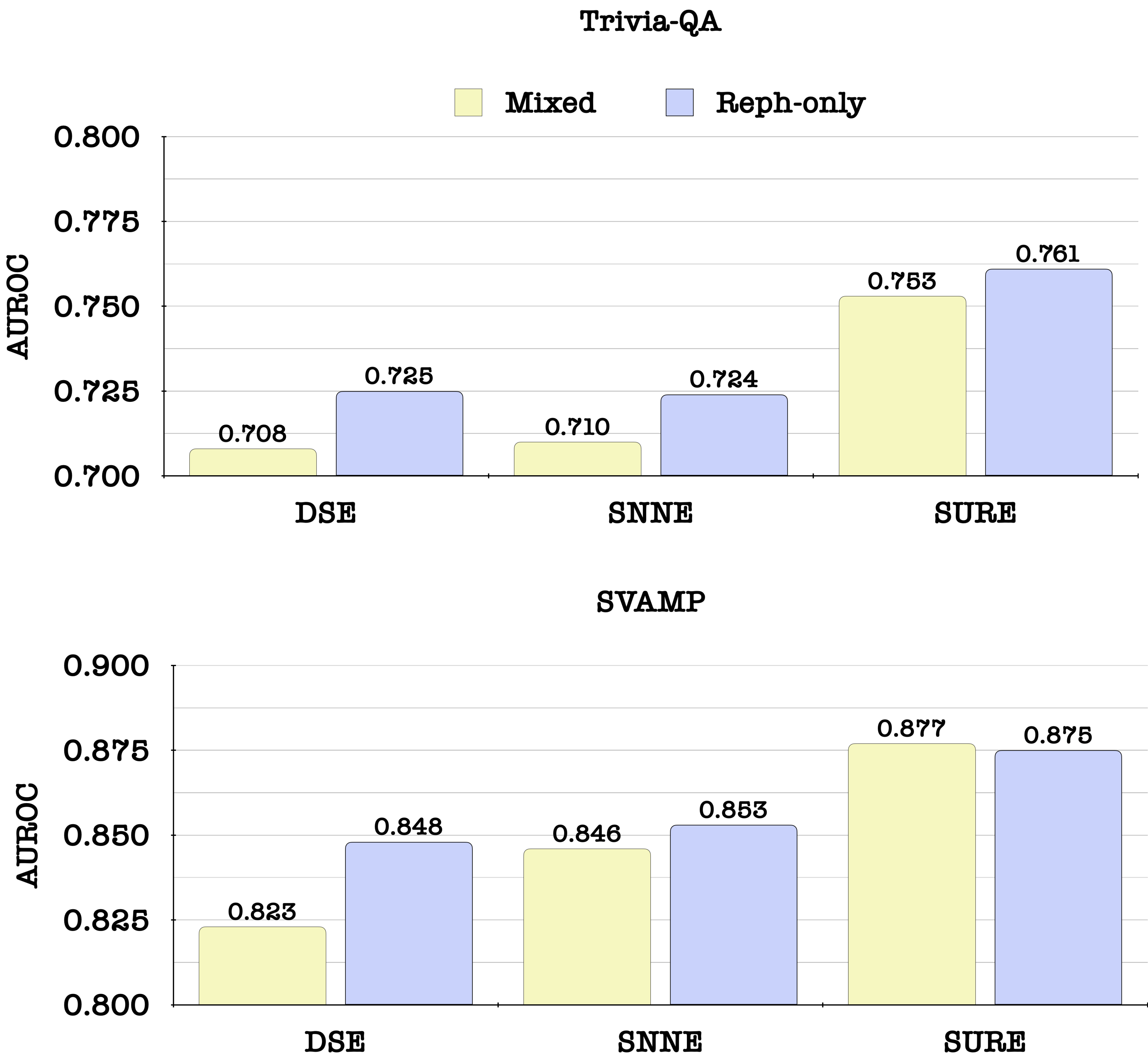}
    \caption{\textcolor{black}{\textbf{Effect of Rephrase-only ablation on all methods.}
Using Rephrase-only inputs for different methods improves them most of the time, however it comes with $2X$ rephrase cost.}}
    \vspace{-0.05in}
    \label{fig:ablation-rephrase-only}
\end{figure}
\begin{table*}[t]
\centering
\caption{Comparison of baseline methods using 13 samples. Results are reported as abstention AUROC ($\text{mean} \pm \text{std}$ over 5 seeds). The second column reports the base task accuracy. The best result in each row is highlighted in \textbf{bold}, and the second-best result is \underline{underlined}.}
\label{tab:baselines_13samples}
\setlength{\tabcolsep}{4pt}
\small

\setlength{\aboverulesep}{0pt}
\setlength{\belowrulesep}{0pt}
\setlength{\heavyrulewidth}{1.2pt}
\setlength{\lightrulewidth}{0.55pt}

\resizebox{0.98\linewidth}{!}{%
\renewcommand{\arraystretch}{1.2}
\begin{tabular}{l c c c c c c c c c}

\specialrule{1.2pt}{0pt}{1.2pt}
\specialrule{0.55pt}{0pt}{0.9pt}

\textbf{Model} & \textbf{Acc.} & \textbf{DSE} & \textbf{SumEigv} & \textbf{Deg} & \textbf{NumSet} & \textbf{LexSim} & \textbf{KLE} & \textbf{SNNE} & \textbf{SURE} \\

\midrule[0.8pt]

\multicolumn{10}{c}{\textbf{Trivia-QA}} \\
\midrule[0.6pt]

Qwen-7B   & 0.495 & 0.7897 $\pm$ 0.0034 & \underline{0.7984 $\pm$ 0.0052} & 0.7979 $\pm$ 0.0048 & 0.7884 $\pm$ 0.0045 & 0.7870 $\pm$ 0.0081 & 0.7733 $\pm$ 0.0076 & 0.7884 $\pm$ 0.0033 & \textbf{0.8060 $\pm$ 0.0141} \\
Llama-8B  & 0.661 & 0.7221 $\pm$ 0.0052 & 0.7203 $\pm$ 0.0057 & 0.7320 $\pm$ 0.0058 & 0.7123 $\pm$ 0.0072 & \underline{0.7367 $\pm$ 0.0069} & 0.7185 $\pm$ 0.0032 & 0.7325 $\pm$ 0.0068 & \textbf{0.7720 $\pm$ 0.0212} \\
Gemma-12B & 0.643 & 0.7084 $\pm$ 0.0102 & 0.7078 $\pm$ 0.0107 & 0.7045 $\pm$ 0.0094 & 0.7094 $\pm$ 0.0100 & \underline{0.7326 $\pm$ 0.0144} & 0.6974 $\pm$ 0.0122 & 0.7182 $\pm$ 0.0142 & \textbf{0.7886 $\pm$ 0.0110} \\
Gemma-27B & 0.755 & 0.6317 $\pm$ 0.0055 & 0.6364 $\pm$ 0.0033 & 0.6353 $\pm$ 0.0030 & 0.6305 $\pm$ 0.0054 & \underline{0.6584 $\pm$ 0.0057} & 0.6290 $\pm$ 0.0064 & 0.6189 $\pm$ 0.0102 & \textbf{0.6991 $\pm$ 0.0165} \\
Llama-70B & 0.767 & \underline{0.7285 $\pm$ 0.0086} & 0.7582 $\pm$ 0.0097 & \textbf{0.7617 $\pm$ 0.0095} & 0.7157 $\pm$ 0.0117 & 0.7419 $\pm$ 0.0071 & 0.7064 $\pm$ 0.0109 & 0.7187 $\pm$ 0.0063 & 0.7145 $\pm$ 0.0087 \\
Qwen-72B  & 0.690 & 0.6988 $\pm$ 0.0125 & \underline{0.7232 $\pm$ 0.0051} & 0.7167 $\pm$ 0.0054 & 0.6779 $\pm$ 0.0087 & 0.6326 $\pm$ 0.0031 & 0.6999 $\pm$ 0.0042 & 0.7006 $\pm$ 0.0087 & \textbf{0.7397 $\pm$ 0.0295} \\

\midrule
\textbf{Avg.} & 0.669 & 0.7132 & \underline{0.7240} & 0.7247 & 0.7057 & 0.7149 & 0.7041 & 0.7129 & \textbf{0.7533} \\

\midrule[0.8pt]

\multicolumn{10}{c}{\textbf{SVAMP}} \\
\midrule[0.6pt]

Qwen-7B   & 0.688 & 0.7152 $\pm$ 0.0271 & 0.7839 $\pm$ 0.0199 & 0.7779 $\pm$ 0.0186 & 0.7838 $\pm$ 0.0279 & 0.7572 $\pm$ 0.0202 & \underline{0.7871 $\pm$ 0.0185} & 0.7745 $\pm$ 0.0297 & \textbf{0.8039 $\pm$ 0.0167} \\
Llama-8B  & 0.672 & 0.6203 $\pm$ 0.0224 & 0.8598 $\pm$ 0.0257 & 0.8617 $\pm$ 0.0267 & 0.8747 $\pm$ 0.0214 & 0.8595 $\pm$ 0.0169 & 0.8626 $\pm$ 0.0191 & \underline{0.8825 $\pm$ 0.0205} & \textbf{0.8897 $\pm$ 0.0261} \\
Gemma-12B & 0.730 & 0.7731 $\pm$ 0.0162 & \underline{0.7979 $\pm$ 0.0157} & 0.7998 $\pm$ 0.0162 & 0.7902 $\pm$ 0.0179 & 0.7940 $\pm$ 0.0166 & 0.7940 $\pm$ 0.0149 & 0.8011 $\pm$ 0.0156 & \textbf{0.8821 $\pm$ 0.0208} \\
Gemma-27B & 0.794 & 0.7482 $\pm$ 0.0148 & 0.7564 $\pm$ 0.0113 & 0.7573 $\pm$ 0.0111 & 0.7404 $\pm$ 0.0140 & 0.7426 $\pm$ 0.0146 & 0.7422 $\pm$ 0.0150 & \underline{0.8039 $\pm$ 0.0123} & \textbf{0.8786 $\pm$ 0.0168} \\
Llama-70B & 0.797 & 0.8788 $\pm$ 0.0112 & 0.8844 $\pm$ 0.0134 & 0.8872 $\pm$ 0.0130 & 0.8885 $\pm$ 0.0079 & 0.8759 $\pm$ 0.0073 & 0.8896 $\pm$ 0.0116 & \underline{0.8982 $\pm$ 0.0115} & \textbf{0.9051 $\pm$ 0.0132} \\
Qwen-72B  & 0.849 & 0.9020 $\pm$ 0.0225 & \textbf{0.9235 $\pm$ 0.0088} & \underline{0.9234 $\pm$ 0.0094} & 0.8837 $\pm$ 0.0221 & 0.8074 $\pm$ 0.0138 & 0.8421 $\pm$ 0.0225 & 0.9065 $\pm$ 0.0098 & 0.9018 $\pm$ 0.0392 \\

\midrule
\textbf{Avg.} & 0.755 & 0.7729 & 0.8343 & 0.8346 & 0.8269 & 0.8061 & 0.8196 & \underline{0.8445} & \textbf{0.8769} \\

\specialrule{0.55pt}{0.9pt}{0pt}
\specialrule{1.2pt}{0pt}{0pt}

\end{tabular}%
}
\end{table*}
\subsection{Matched inference-budget analysis}
\label{appendix:matched-inference-comparison}
\textcolor{black}{Our default SURE configuration generates $M=3$ low-temperature anchors and $N=10$ high-temperature probes, resulting in $13$ total LLM generations per input. In contrast, several sampling-based baselines are commonly evaluated with $10$ generations. To verify that the observed improvements are not simply due to this larger generation budget, we rerun the text-QA baselines using $N=13$ samples. As reported in Table~\ref{tab:baselines_13samples}, SURE retains its performance advantage even when the baselines are given the same number of LLM calls, suggesting that the improvement comes from the cross-temperature anchor--probe construction rather than from additional sampling alone.}
\subsection{\textcolor{black}{Temperature sensitivity analysis}}
\label{app:temperature_sensitivity}

\textcolor{black}{
We evaluate the sensitivity of SURE to the choice of low-temperature anchor sampling temperature $T_L$ and high-temperature probe sampling temperature $T_H$. Figure~\ref{fig:temperature_sensitivity} reports abstention AUROC on SVAMP and Trivia-QA for $T_L \in \{0.1,0.2,0.3,0.4\}$ and $T_H \in \{1.0,1.2,1.5,2.0\}$.
SURE is relatively insensitive to the exact choice of $T_L$ within the low-temperature range: for a fixed $T_H$, changing $T_L$ produces only small variations in AUROC. This supports the role of low-temperature samples as stable semantic anchors. Second, performance is more sensitive to the probe temperature $T_H$. Moderate probe temperatures, especially $T_H=1.0$ and $T_H=1.2$, perform best, while very high temperatures such as $T_H=2.0$ degrade performance substantially. This suggests that the probes should introduce semantic diversity without becoming excessively noisy. We use the $T_L=0.1$ and $T_H=1.0$ in all results in the paper, without tuning.
}
\begin{figure}[t!]
    \centering
    \includegraphics[width=1.0\linewidth, height=0.23\textwidth]{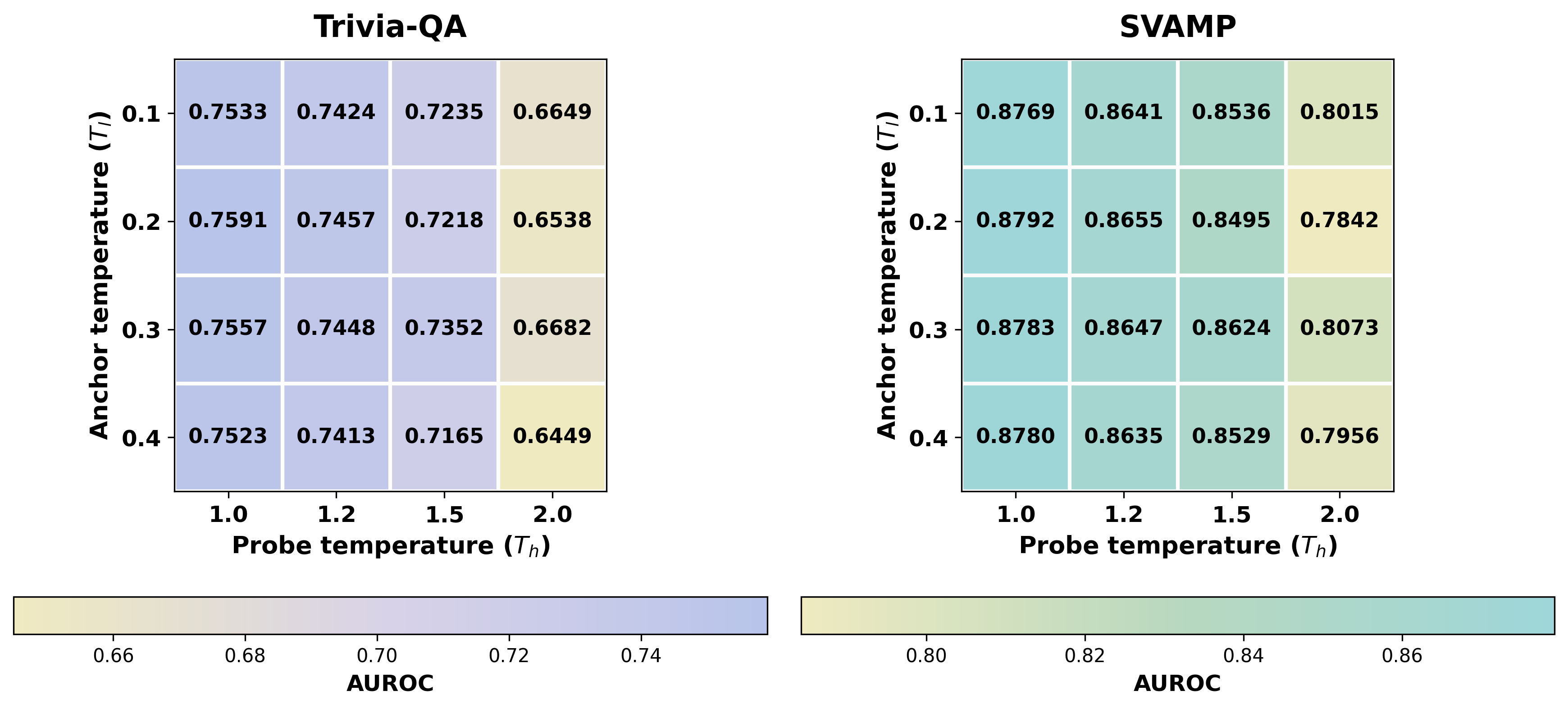}
    \caption{\textbf{Temperature sensitivity analysis for SURE.} Abstention AUROC for different anchor temperatures ($T_l$) and probe temperatures ($T_h$) on Trivia-QA and SVAMP. Performance is relatively stable across low-temperature anchor choices, while very high probe temperatures, especially $T_h=2.0$, reduce AUROC.}
\label{fig:temperature_sensitivity}
    \vspace{-0.1in}
\end{figure}
\subsection{\textcolor{black}{Input augmentation prompts and implementation details}}
\label{app:input_aug_prompts}

\textcolor{black}{
This section provides the implementation details for the equivalent-input augmentations used in SURE. For textual inputs, we use LLM-based rephrasing to create semantically equivalent versions of the original question. The rephrased questions are used only to generate high-temperature probe responses; the low-temperature anchors are generated from the original input. This preserves the interpretation of anchors as stable responses to the original question, while allowing the probes to test semantic stability under both stochastic decoding and meaning-preserving input variation. The prompts used to create semantics-preserving textual augmentations using Llama-3.1-8B-Instruct (for text-QA) and Gemini-2.5-Flash (for multilingual QA) are shown below. For multimodal QA, we apply both textual and visual input augmentations. The textual component rephrases the question while keeping the associated image fixed. The visual component perturbs the image while preserving its semantic content. We use lightweight image transformations including mean blur with kernel size $5$, affine rotation of $\pm 10^\circ$ around the image center, and Gaussian pixel noise sampled from $\mathcal{N}(0,10)$ followed by clipping to the valid pixel range $[0,255]$. These transformations are intentionally mild: they are designed to probe robustness to small input changes without changing the answer semantics. Together, these augmentations allow SURE to test whether high-temperature probe responses remain aligned with stable low-temperature anchors under equivalent input variations.
}
\begin{tcolorbox}[
    colback=gray!8,
    colframe=gray!55,
    title=\textbf{Prompt for Gemini: Multilingual Rephrasing},
    fonttitle=\bfseries,
    boxrule=0.6pt,
    arc=2pt,
    left=6pt,right=6pt,top=6pt,bottom=6pt
]
\small
\textbf{System instruction:}\\
You are a multilingual question rephrasing assistant. Your task is to rephrase questions in multiple languages using \textbf{only} the information present in the original questions. You must \textbf{never} add external knowledge.

\vspace{4pt}
\textbf{Strict rules:}
\begin{enumerate}[leftmargin=*,nosep]
    \item Use only words and concepts from the original question; do not add any external knowledge.
    \item Do not add descriptors that hint at the answer.
    \item Keep all specific names, titles, dates, and key terms from the original exactly as written.
    \item Each rephrase should ask for the exact same information---no more, no less.
    \item Maintain the same language style and formality level for each language.
    \item Generate rephrasings for each language.
\end{enumerate}

\vspace{4pt}
\textbf{Input template:}
\begin{quote}
\ttfamily
Original questions: \{questions\_text\}
\end{quote}

\textbf{Output format:}
\begin{quote}
\ttfamily
[english] \\
rephrase here \\[2pt]
[chinese] \\
rephrase here \\[2pt]
[japanese] \\
rephrase here \\[2pt]
[french] \\
rephrase here
\end{quote}

\textbf{Instruction:} Provide the output in exactly the format above.
\end{tcolorbox}
\begin{tcolorbox}[
    colback=gray!8,
    colframe=gray!55,
    title=\textbf{Prompt for Llama: English Rephrasing},
    fonttitle=\bfseries,
    boxrule=0.6pt,
    arc=2pt,
    left=6pt,right=6pt,top=6pt,bottom=6pt
]
\small
\textbf{System instruction:} \\
You are a question rephrasing assistant. You rephrase questions using \textbf{only} the information present in the original question. You must \textbf{never} add external knowledge.

\vspace{4pt}
\textbf{User prompt:}
\begin{quote}
\ttfamily
Rephrase the following question.
\end{quote}

\textbf{Strict rules:}
\begin{enumerate}[leftmargin=*,nosep]
    \item Use only words and concepts from the original question; do not add any external knowledge.
    \item Do not add descriptors that hint at the answer.
    \item Keep all specific names, titles, and key terms from the original exactly as written.
    \item Each rephrase should ask for the exact same information---no more, no less.
\end{enumerate}

\vspace{4pt}
\textbf{Input template:}
\begin{quote}
\ttfamily
Original question: \{question\}
\end{quote}

\textbf{Output instruction:}
\begin{quote}
\ttfamily
Output exactly the rephrased question.
\end{quote}
\end{tcolorbox}
\subsection{Recruiting human subjects}
\label{appendix:human-study}
Participants for the human-evaluation studies were recruited through Prolific, an online participant-recruitment platform. For the multilingual evaluations, we used Prolific's language-based prescreening filters to recruit participants. Participants were compensated through Prolific according to the platform's standard compensation rates. Recruitment, communication, and payment were handled through the platform. Prolific-assigned participant identifiers were used throughout the studies, and no directly identifying information, such as participants' names or contact details, was collected by the authors. Participant identities therefore remained anonymous to the authors.
\subsection{\textcolor{black}{Evaluation of greedy responses}}
\label{appendix:evaluation}
\textcolor{black}{Evaluation of greedy responses for the $3$ tasks of textual QA, multilingual QA and multimodal QA is performed as below:\\
\textbf{Textual QA:} The evaluation of greedy responses is performed using the SQuAD F1 score metric, followed by using $0.5$ threshold to make it binary as suggested in the prior works~\cite{farquhar2024detecting,nguyen2025beyond}.\\
\textbf{Multimodal QA:} For this, the correctness is computed based on agreement with the multiple human annotations provided in the dataset, like prior works~\cite{antol2015vqa, okvqa}.\\
\textbf{Multilingual QA:} We use \texttt{Gemini-2.5-Flash} as a judge for evaluating responses in multilingual QA. The prompt used for the purpose is stated below,}
\begin{tcolorbox}[
    colback=gray!8,
    colframe=gray!55,
    title=\textbf{Prompt for Gemini-2.5-Flash: Multilingual Correctness Judging},
    fonttitle=\bfseries,
    boxrule=0.6pt,
    arc=2pt,
    left=6pt,right=6pt,top=6pt,bottom=6pt
]
\small
\textbf{System instruction:} \\
You are evaluating science questions across multiple languages. Determine if each predicted answer matches the ground truth. Be flexible with extra explanation, but the core answer must be correct.

\vspace{4pt}
\textbf{Output format:}
\begin{quote}
\ttfamily\raggedright
For each language below, decide if the predicted answer is correct (1) or incorrect (0). Return ONLY a JSON object mapping each language code to 1 or 0.

\vspace{4pt}
Example response: \{"en": 1, "zh": 0, "ja": 1, "fr": 0, "th": 1\}
\end{quote}

\textbf{Input template:}
\begin{quote}
\ttfamily\raggedright
=== \{lang\} ===

\vspace{4pt}
Question: \{question\}

\vspace{4pt}
Ground truth: \{ground\_truth\}

\vspace{4pt}
Predicted: \{first\_greedy\_prediction\}
\end{quote}

\textbf{Output instruction:}
\begin{quote}
\ttfamily
Return ONLY the JSON object. No extra text.
\end{quote}
\end{tcolorbox}
\subsection{Area Under the Risk--Coverage Curve}
\label{appendix:aurc-results}

In addition to the abstention AUROC, we evaluate selective-prediction performance using the Area Under the Risk--Coverage Curve (AURC). For each uncertainty estimator, predictions are ordered from the most to the least confident. At each coverage level, the risk is computed as the error rate among the retained predictions. A lower AURC value is better.
Table~\ref{tab:aurc_results} reports AURC averaged across the six text-QA models on SVAMP and TriviaQA. BiG-SURE achieves the lowest AURC on both datasets. These results are consistent with the abstention-AUROC findings.
\begin{table}[t]
    \centering
    \small
    \setlength{\tabcolsep}{5pt}
    \caption{AURC averaged across the six text-QA models. 
    Values are reported as mean $\pm$ standard deviation across models. 
    Lower values indicate better selective-prediction performance.}
    \label{tab:aurc_results}
    \resizebox{0.8\linewidth}{!}{%
    \begin{tabular}{lcc}
        \toprule
        \textbf{Method} 
        & \textbf{SVAMP AURC $\downarrow$}
        & \textbf{TriviaQA AURC $\downarrow$} \\
        \midrule
        DSE      & $0.1097 \pm 0.0441$ & $0.2244 \pm 0.0577$ \\
        SumEigV  & $0.1169 \pm 0.0487$ & $0.2209 \pm 0.0560$ \\
        Deg      & $0.1169 \pm 0.0487$ & $0.2209 \pm 0.0562$ \\
        NumSet   & $0.1104 \pm 0.0439$ & $0.2230 \pm 0.0596$ \\
        LexSim   & $0.1146 \pm 0.0394$ & $\underline{0.2126 \pm 0.0574}$ \\
        KLE      & $0.1212 \pm 0.0398$ & $0.2176 \pm 0.0619$ \\
        SNNE     & $\underline{0.1017 \pm 0.0391}$ & $0.2328 \pm 0.0535$ \\
        \midrule
        \textbf{BiG-SURE}
        & $\mathbf{0.0839 \pm 0.0328}$
        & $\mathbf{0.1975 \pm 0.0520}$ \\
        \bottomrule
    \end{tabular}
    }
\end{table}
\subsection{Failure analysis of SURE}
\label{appendix:failure_analysis}
We analyze representative failure cases in which textual semantic agreement can produce misleading uncertainty estimates. We discuss two distinct failure modes: (i) \textbf{Case I:} The LLM mispredicts and produces different semantically equivalent forms of the same incorrect responses during sampling, and (ii) \textbf{Case II:} The semantic-similarity backend incorrectly identifies non-equivalent responses as equivalent.
\begin{figure}[t]
    \centering
    \includegraphics[width=0.7\columnwidth]{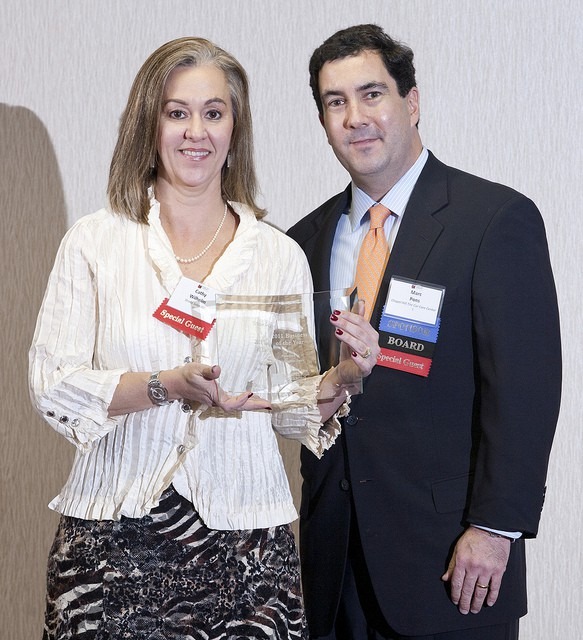}
    \caption{A representative visual-QA failure involving semantically consistent but incorrect responses. For the displayed image, the VLM produces \textit{``Louis Philippe''} and \textit{``Louis Philippe Suits''}, whereas the acceptable reference answers include \textit{Brooks Brothers}, \textit{Giorgio Armani}, and \textit{Gucci}. The high NLI entailment score ($0.9863$) correctly reflects agreement between the generated responses, but cannot determine whether they are visually grounded.}
    \label{fig:consistent_incorrect_vqa}
\end{figure}
\paragraph{Case I:} Figure~\ref{fig:consistent_incorrect_vqa} shows an example from visual QA. For the question \textit{``What brand of suit is the man in the image wearing?''}, the acceptable reference answers include \textit{Brooks Brothers}, \textit{Giorgio Armani}, and \textit{Gucci}. However, the VLM repeatedly generates \textit{``Louis Philippe''} and \textit{``Louis Philippe Suits''}. The NLI backend assigns these responses an entailment score of $0.9863$ and estimates very low uncertainty. This is the case of consistent incorrect predictions.
This should not be interpreted as an NLI-backend error: the two generated responses are semantically equivalent, and therefore the high entailment score is reasonable. Instead, it illustrates a fundamental limitation of response-consistency-based uncertainty estimation. This limitation is especially important in visual QA because the current semantic backend compares only the textual responses and does not examine whether they are grounded in the image.

\paragraph{Case II}
A separate failure mode occurs when the NLI backend assigns high entailment to responses that are not semantically equivalent. For example, for a Chinese SciQ question asking what type of energy is released when iron filings react with air, the responses corresponding to \textit{``heat''} and \textit{``combustion''} receive an entailment score of $0.8453$. Although the concepts are related, combustion is a process rather than the requested form of energy, and the two answers are not interchangeable.
We observe a similar issue for numerical and frequency-based answers. For a Japanese SciQ question asking how often partial lunar eclipses occur, the responses corresponding to \textit{``more than once per month''} and \textit{``multiple times per year''} receive an entailment score of $0.9895$, despite expressing substantially different frequencies.

\noindent The two failure modes have different origins. Semantically consistent but incorrect predictions cannot be resolved by improving textual NLI alone and instead require grounding-aware comparison with the original input, including the image in multimodal settings. Semantic-backend errors may be mitigated through stronger multilingual and numerical entailment models, and calls for improving such models. Although the same semantic backend is used for BiG-SURE and the similarity-based baselines in our experiments, such errors does not guarantee that scorer errors affect all methods equally; as their impact may vary across methods. Developing modality-aware and more reliable semantic-agreement measures therefore remains an important direction for future work.
\subsection{Use of AI assistants}
We used an AI assistant- ChatGPT~\cite{achiam2023gpt} solely as a language-editing aid. Its use was limited to identifying grammatical errors and typographical mistakes and suggesting localized revisions to improve clarity, and fluency. It was not used to generate ideas, formulate hypotheses, or design methods. All technical content, arguments, analyses, and conclusions were developed and verified by the authors, who also reviewed and approved all final edits.

\end{document}